\documentclass[final]{colt2026} 

\usepackage{mathtools}
\usepackage{enumitem}
\usepackage{algorithm}
\usepackage{algorithmic}
\usepackage{thm-restate}

\newenvironment{proofsketch}{\paragraph{Proof sketch}}{\hfill$\blacksquare$}

\newcommand{\FTPL}{\textsc{Ftpl}}
\newcommand{\FTRL}{\textsc{Ftrl}}
\newcommand{\GBPA}{\textsc{Gbpa}}
\newcommand{\scrible}{\textsc{SCRiBLe}}
\newcommand{\BanditsGBPA}{\textsc{Bandits-Gbpa}}
\newcommand{\SCFTPL}{\textsc{Sc-Ftpl}}

\newcommand{\R}{\mathbb{R}}
\newcommand{\N}{\mathbb{N}}
\newcommand{\E}{\mathbb{E}}
\newcommand{\B}{\mathbb{B}}
\renewcommand{\P}{\mathbb{P}}
\renewcommand{\S}{\mathbb{S}}
\renewcommand{\O}{\mathcal{O}}
\newcommand{\calr}{\mathcal{R}}
\newcommand{\cald}{\mathcal{D}}
\newcommand{\calp}{\mathcal{P}}

\DeclareMathOperator*{\argmax}{arg\,max}
\DeclareMathOperator*{\argmin}{arg\,min}
\newcommand{\sgn}{\operatorname{sgn}}
\newcommand{\inter}{\operatorname{int}}
\newcommand{\dom}{\operatorname{dom}}
\newcommand{\spa}{\operatorname{span}}
\newcommand{\Pperp}{P_{\theta^\perp}}
\newcommand{\Ima}{\operatorname{Im}}
\newcommand{\Cov}{\operatorname{Cov}}
\newcommand{\diag}{\operatorname{Diag}}
\newcommand{\tr}{\operatorname{tr}}
\newcommand{\inc}{\subseteq}
\newcommand{\sca}[2]{\langle #1,#2 \rangle}

\title[Self-Concordant Perturbations for Linear Bandits]{Self-Concordant Perturbations for Linear Bandits}
\usepackage{times}

\coltauthor{%
 \Name{Lucas L{\'e}vy} \Email{lucas.levy@inria.fr}\\
 \addr {\'E}cole Polytechnique and University of Oxford
 \AND
 \Name{Jean-Lou Valeau} \Email{jeanlou.valeau@ensae.fr}\\
 \addr ENSAE Paris and University of Oxford
 \AND
 \Name{Arya Akhavan} \Email{arya.akhavan@stats.ox.ac.uk}\\
 \addr University of Oxford and {\'E}cole Polytechnique
 \AND
 \Name{Patrick Rebeschini} \Email{patrick.rebeschini@stats.ox.ac.uk}\\
 \addr University of Oxford
 }

\begin{document}

\maketitle

\begin{abstract}
We consider the adversarial linear bandits setting and present a unified algorithmic framework that bridges Follow-the-Regularized-Leader (\textsc{Ftrl}) and Follow-the-Perturbed-Leader (\textsc{Ftpl}) methods, extending the known connection between them from the full-information setting. Within this framework, we introduce self-concordant perturbations, a family of probability distributions that mirror the role of self-concordant barriers previously employed in the \textsc{Ftrl}-based \textsc{SCRiBLe} algorithm. Using this idea, we design a novel \textsc{Ftpl}-based algorithm that combines self-concordant regularization with efficient stochastic exploration. Our approach achieves a regret of $\mathcal{O}(d\sqrt{n \ln n})$ on both the $d$-dimensional hypercube and the $\ell_2$ ball. On the $\ell_2$ ball, this matches the rate attained by \textsc{SCRiBLe}. For the hypercube, this represents a $\sqrt{d}$ improvement over these methods and matches the optimal bound up to logarithmic factors.
\end{abstract}


\section{Introduction}

We study \emph{online linear optimization under bandit feedback}, where an agent sequentially selects actions based on the information available from previous steps. The loss of each action varies over time and is unknown to the agent at the moment of decision. At the end of each round, only the loss associated to the chosen action is revealed. The agent's aim is to minimize the \emph{regret}, the difference between the agent's cumulative loss and that of the best fixed action in hindsight.
In this setting, two prominent families of algorithms have been widely studied: \emph{Follow-the-Regularized-Leader} (\FTRL) \citep{abernethy2009competing} and \emph{Follow-the-Perturbed-Leader} (\FTPL) \citep{kalaivampala}.
In the full-information setting, i.e., when the whole loss function is observed by the agent at the end of each round, \citet{abernethy14, abernethy16} introduced a general framework, \emph{Gradient-Based Prediction Algorithm} (\GBPA), which encompasses both \FTRL\space and \FTPL\space and enables a unified analysis.

Many algorithms designed for the bandit setting adapt techniques from the full-information setting but rely on estimators of the loss function rather than its exact value. Such estimators must be constructed from limited feedback, which makes \emph{exploration}, deliberately selecting potentially suboptimal actions to gather information, essential \citep{cesabianchilugosi}. Stochastic exploration mechanisms include playing from a fixed exploration distribution with small probability \citep{auer02}, or stochastically shifting the solution given by \FTRL\space \citep{flaxman}.

An important instance of \FTRL, \emph{Self-Concordant Regularization in Bandit Learning} (\scrible) was introduced by \citet{abernethy12}. The algorithm relies on self-concordant barriers \citep{nemirovskii}, a class of convex functions that exhibit a controlled divergence at the boundary of their domain and satisfy a form of local strong convexity with respect to the local norm induced by their Hessian. These properties make self-concordant barriers well adapted to the geometry of the action set and explain their use as regularizers in \scrible.
This algorithm attains a regret bound of $\O(d^{3/2}\sqrt{n \ln n})$ for arbitrary convex body action sets $K\subset\R^d$, where $n$ is the number of rounds, or \emph{horizon}. When the action set is the $d$-dimensional $\ell_2$ ball, the bound improves to $\O(d\sqrt{n \ln n})$.
Moreover, whenever a closed-form expression of the barrier is available, which holds for many natural convex sets, the algorithm admits an implementation with per-round computational complexity polynomial in $d$. 

For linear bandits, \citet{dani_pricebanditinfo} established a minimax lower bound of $\Omega(d\sqrt{n})$ on the regret for arbitrary convex bodies.
Subsequent work closed this gap in terms of rates: \citet{bubeck2012towards} proposed an \FTRL-based algorithm achieving $\O(d\sqrt{n\ln n})$ regret in general, and introduced \textsc{Osmd}, which attains sharper bounds of $\O(d\sqrt{n})$ for the $d$-dimensional hypercube and $\O(\sqrt{dn\ln n})$ for the $\ell_2$ ball.
More recently, \citet{hazan2016volumetric} gave an efficient algorithm achieving $\O(d\sqrt{n\ln n})$ regret for arbitrary convex bodies using spanning ellipsoids for exploration, and \cite{hoeven2018many} achieved a similar rate using an exploration scheme based on John's ellipsoid.
Despite these advances, existing optimal-rate algorithms rely on carefully tailored preprocessing schemes, and do not arise naturally from the self-concordant barrier framework. This motivates the question of whether self-concordant–based methods can be extended beyond \scrible\space to achieve near-optimal regret while retaining a simple, perturbation-based exploration mechanism.

We study this question through \FTPL, which selects its action by solving a randomly perturbed linear program. While \FTPL\space can be reduced to \FTRL\space in the full-information linear setting \citep{abernethy14}, the distinction is crucial in bandits: \FTRL\space requires an explicit exploration mechanism, while \FTPL\space is inherently stochastic, and samples its action from the extreme points of the action set. This observation motivates \FTPL-based methods as a natural vehicle for efficient exploration, provided the perturbations replicate the effects of self-concordant regularization.

We make the following contributions.
\begin{enumerate}
\item We introduce \BanditsGBPA, a unifying framework for linear bandit algorithms that includes both \FTRL- and \FTPL-based approaches. \BanditsGBPA\space extends the \GBPA\space framework of \citet{abernethy14, abernethy16}, originally developed for the full-information setting, to the bandit setting.
While this was also the goal of \cite{abernethy2015fighting}, their work is limited to the multi-armed bandit case, i.e., when the action set is finite.
Unlike in the full-information setting, our framework does not imply the inclusion of \FTPL\space within \FTRL, but it provides a conceptual structure allowing a common analysis.

\item Building on this framework, we introduce the \emph{Self-Concordant FTPL} (\SCFTPL) algorithm, which incorporates both self-concordant regularization, and an \FTPL\space sampling scheme. A central element of our approach is the identification and analysis of a family of perturbation distributions that we call \emph{self-concordant perturbations}. These distributions replicate, within an \FTPL\space framework, the properties that self-concordant barriers provide in \FTRL, while defining a sampling scheme enabling more exploration.

\item We construct self-concordant perturbations for two canonical action sets, and we analyze their impact on \SCFTPL. 
These constructions serve not merely as domain-specific algorithmic designs, but as concrete examples showing how perturbations can recover self-concordant barrier properties.
When the action set is the $d$-dimensional hypercube, we show that \SCFTPL \space produces lower-variance loss vector estimates, and achieves a regret of $\O(d\sqrt{n \ln n})$, matching the minimax lower bound established by \citet{dani_pricebanditinfo} up to a logarithmic factor. This represents an improvement of a factor $\sqrt{d}$ over \scrible\space and matches the performance of \textsc{Osmd}\space\citep{bubeck2012towards}, without requiring an additional domain-specific sampling scheme. For the $\ell_2$ ball however, \SCFTPL \space achieves suboptimal regret in $\O(d\sqrt{n\ln n})$, matching the performance of \scrible. 
\end{enumerate}

Our proofs combine three types of techniques. First, we leverage standard results and decomposition methods from the bandit literature \citep{abernethy2009competing, bubeck_cesabianchi_12} to handle the core regret arguments. Second, we perform careful computations to identify suitable self-concordant barriers and to bound the norms of the estimators. While sometimes intricate, these are mostly technical. Finally, the analysis of \SCFTPL\space on the $\ell_2$ ball requires a novel argument to handle the absence of a uniform almost-sure bound on the local norms. We control the growth of the cumulative estimator via a carefully constructed supermartingale and stopping time arguments. This technique may be of independent interest for other sequential decision making problems with similar challenges.

The rest of the paper is organized as follows. Section~\ref{section:2} formalizes the linear bandits setting. Then, in Section~\ref{section:main_results}, we present our main results: the \SCFTPL\space algorithm, self-concordant perturbations, and regret bounds on the hypercube and $\ell_2$ ball. In Section~\ref{section:gbpa}, we present the unifying \BanditsGBPA\space framework. Section~\ref{section:from_scr_to_scp} reviews key properties of self-concordant barriers and motivates \SCFTPL’s sampling and estimation schemes. In Section~\ref{section:specific_action_sets}, we detail the analysis of \SCFTPL\space for the hypercube and the $\ell_2$ ball. Finally, in Section~\ref{section:conclusion}, we discuss some open research directions.

\section{Problem Setting} \label{section:2}

We now formalize the \emph{Adversarial Linear Bandits} framework, mentioned in the introduction. An instance of this problem is defined by an \emph{horizon} $n\in \N$, an \emph{action set} $K\subset \R^d$, and an unknown sequence of loss vectors $y_1,\ldots,y_n \in \R^d$. In each round $t\in[n]\coloneq \{1,...,n\}$, the learner chooses, possibly at random, an action $a_t\in K$. Then, they observe the loss associated to their chosen action $\sca{y_t}{a_t}$. The learner's decision can depend on an exogenous source of randomness and the history of previous feedback: $a_1, \sca{y_1}{a_1},\ldots, a_{t-1}, \sca{y_{t-1}}{a_{t-1}}$, but not on the current loss vector $y_t$.

On top of that, we make the following assumptions. First, the action set $K$ is a convex body, i.e., a compact convex set with non-empty interior. Moreover, the losses are bounded: for all loss vectors $y$ and action $a\in K$, $|\sca{y}{a}|\le 1$.

The quantity of interest is the difference between the learner's cumulated loss and the cumulated loss of the best action in-hindsight, called the regret. We define the learner's \emph{regret with respect to some competitor $u\in K$} as
\begin{equation}
    R_n(u)\coloneq\E\Big[\sum_{t=1}^n\sca{y_t}{a_t}\Big]-\sum_{t=1}^n\sca{y_t}{u}\,,
\end{equation}
where the expectation is taken with respect to the randomness in the actions of the learner.
The learner's \emph{regret} is then defined as
$R_n\coloneq\sup_{u\in K}R_n(u)$. 

\section{Main Results}\label{section:main_results}

In this section, we introduce our main contribution, the \emph{Self-Concordant Follow-the-Perturbed-Leader} (\SCFTPL) algorithm, along with its regret guarantees for the hypercube and the $\ell_2$ ball.
We begin by defining \emph{self-concordant perturbations}. Their definition assumes familiarity with $\vartheta$-self-concordant barriers. Additional details and properties of such barriers are deferred to Section~\ref{section:self_concordant_barriers}.

\begin{definition}[Self-Concordant Perturbation] \label{def:scp}
Let $K\subset \R^d$ be a convex body, $\vartheta\ > 0$, and $\cald$ be a probability distribution on $\R^d$, absolutely continuous with respect to the Lebesgue measure. Let $\phi_K:\R^d\to\R$ be the support function of $K$. We say that $\cald$ is a \emph{$\vartheta$-self-concordant perturbation} for $K$ if there exists $\calr$ a $\vartheta$-self-concordant barrier on $K$ such that
\[\nabla\calr^*(\theta)=\E_{\xi\sim\cald}[\nabla\phi_K(\theta+\xi)]\quad \text{for all }\theta\in\R^d\,,\]
where $\calr^*$ is the Fenchel conjugate of $\calr$. In this case, we say that $\cald$ \emph{replicates} $\calr$.
\end{definition}
Note that this definition involves the derivative of $\phi_K$, the support function of $K$, which is not necessarily defined on $\R^d$. However, because $\phi_K$ is convex, it is differentiable almost everywhere. Moreover, $\cald$ is absolutely continuous, so we have that $\P_{\xi\sim\cald}(\phi_K \text{ differentiable in } \theta+\xi)=1$ for all $\theta\in \R^d$, and thus $\E_{\xi\sim\cald}[\nabla \phi_K(\theta+\xi)]$ is well-defined.

For additional intuition on the behavior of self-concordant perturbations, we note in Appendix~\ref{appendix:heavy_tailed} that they are heavy-tailed.

\paragraph{The Self-Concordant FTPL Algorithm}We now define the \emph{Self-Concordant FTPL} (\SCFTPL) algorithm. The learner maintains a cumulative estimate of the loss vectors $\hat Y_{t-1}=\sum_{s=1}^{t-1}\hat y_s$, initialized at 0. At each round $t\in[n]$, the algorithm follows a FTPL scheme: it samples $\xi_t\sim \cald$, where $\cald$ is a self-concordant perturbation, and selects
\[a_t=\argmin_{a\in K} \sca{a}{\eta\,\hat Y_{t-1}-\xi_t}\,.\]
After observing the scalar loss $\sca{y_t}{a_t}$, the learner constructs an estimator $\hat y_t$ of $y_t$, using the second-moment matrix $Q_t=\E_{t-1}[a_ta_t^\top]$, where $\E_{t-1}[\cdot]$ is the conditional expectation given past actions $a_1,\ldots,a_{t-1}$. The full procedure is specified in Algorithm~\ref{alg:scftpl}.

\begin{algorithm}[ht]
	\caption{Self-Concordant \textsc{Ftpl}\space(\SCFTPL)}
	\label{alg:scftpl}
	\begin{algorithmic}[1]
            \REQUIRE A $\vartheta$-self-concordant perturbation $\cald$ for $K$ and a learning rate $\eta>0$
            \STATE Set $\hat Y_0= 0$.
            \FOR{$t=1, \ldots, n$}
            \STATE Sample $\xi_t\sim \cald$ independently from the past. 
            \STATE Play action $a_t=\argmin_{a\in K} \sca{a}{\eta\,\hat Y_{t-1}-\xi_t}$.
            \STATE Receive punctual loss $\sca{y_t}{a_t}$.
            \STATE Compute estimator $\hat y_t=Q_t^{-1}a_t\sca{y_t}{a_t}$, where $Q_t\coloneq\E_{t-1}[a_ta_t^\top ]$.
            \STATE Update $\hat Y_t=\hat Y_{t-1}+\hat y_t$.
            \ENDFOR
	\end{algorithmic}
\end{algorithm}

As we will see in Section~\ref{section:from_scr_to_scp}, this construction inherits properties of self-concordant barriers used by \scrible, while introducing a new sampling mechanism: \SCFTPL\space draws actions that concentrate on extreme points of $K$, thus exploring more effectively the available action space.

\paragraph{Performance on the Hypercube and $\ell_2$ Ball} We analyze \SCFTPL\space in details on two canonical action sets: the hypercube $[-1,1]^d$ and the unit $\ell_2$ ball $\B^d\coloneq\{x\in\R^d:\|x\|\le 1\}$. In both cases, we construct explicit self-concordant perturbations and derive the following regret guarantees.

\begin{theorem}[Regret of \SCFTPL\space on the Hypercube]\label{thm:regret_scftpl_hypercube}
Let $\cald$ be the $d$-self-concordant perturbation for $[-1,1]^d$ defined in Proposition~\ref{prop:scperturbation_hypercube_existence}. Assume that $\frac{n}{\ln n}\ge 36\, d$. Then, \SCFTPL\space on the hypercube with perturbation distribution $\cald$ and learning rate $\eta=\sqrt{\frac{3\ln n}{n}}$ has regret bounded by
\begin{equation*}
    R_n\le \frac{2\sqrt{3}}{3} d\sqrt{n \ln n}+2\,.
\end{equation*}
\end{theorem}
This result highlights the main advantage of our approach. While the previous \scrible\space approach suffers a regret of $\O(d^{3/2}\sqrt{n})$ on the hypercube due to the suboptimal exploration, our \FTPL-based approach improves this dependence by a factor of $\sqrt{d}$. This matches the minimax optimal rate for linear bandits on the hypercube up to logarithmic factors.

\begin{theorem}[Regret of \SCFTPL\space on the $\ell_2$ Ball]\label{thm:regret_scftpl_ball}
Let $\cald$ be the $1$-self-concordant perturbation for $\B^d$ defined in Proposition~\ref{prop:scperturbation_ball_existence}. Assume that $\frac{n}{\ln n}\ge \max(2d^2,64)$. Then, there exists a universal constant $C>0$ such that \SCFTPL\space on the $\ell_2$ ball with perturbation distribution $\cald$ and learning rate $\eta=\frac{1}{d}\sqrt{\frac{4\ln n}{5\,n}}$ has regret bounded by
\begin{equation*}
    R_n\le \sqrt{5}\,d\sqrt{n\ln n}+2+C\,\frac{\ln^3n}{d^2}\,.
\end{equation*}
\end{theorem}
A more detailed analysis of \SCFTPL\space in this particular case, as well as the construction of the self-concordant perturbations appear in Section~\ref{section:specific_action_sets}.

\section{Gradient-Based Prediction}\label{section:gbpa}

In this section, we present the unified \GBPA\space framework introduced by \citet{abernethy14} in the full-information setting, and then define the \BanditsGBPA\space algorithm.

\subsection{Follow-the-Leader Style Algorithms}

First, we place ourselves in the \emph{full-information setting}. This setting is different from the bandit setting introduced in Section~\ref{section:2} in that the whole loss vector $y_t$ is observed at the end of round $t$, instead of the punctual loss $\sca{y_t}{a_t}$. For all $t\ge 1$, let $Y_{t-1}\coloneq\sum_{s=1}^{t-1} y_s$ be the cumulative loss before time $t$.

An instance of \FTRL\space is defined by a convex \emph{regularizer} $\psi:K\to\R$\footnote{In the literature, the regularizer is usually scaled by a \emph{learning rate} $\eta>0$ \citep{orabona23}. The tuning of the learning rate is crucial in the analysis of \FTRL\space in order to enjoy sublinear regret. However, for the sake of brevity, we do not include it in this first part and simply consider that it is included in the regularizer $\psi$.}. At each time $t\in[n]$, the action chosen by \FTRL\space solves the penalized optimization problem
\begin{equation}
    a_t\in \argmin_{a\in K}\left\{\sca{a}{Y_{t-1}}+\psi(a)\right\}\,,
\end{equation}
where ties are resolved arbitrarily.
\emph{Follow-the-Perturbed-Leader} (\FTPL) adds regularization via a stochastic perturbation \citep{kalaivampala}. Given a probability distribution $\cald$ on $\R^d$ called the \emph{perturbation distribution}, at each time $t\in[n]$, \FTPL\space selects the action
\begin{equation}\label{eq:ftpl}
    a_t\in\argmin_{a\in K}\sca{a}{Y_{t-1}-\xi_t}\,,
\end{equation}
where $\xi_t\sim\cald$ is independent from the past, and ties are resolved arbitrarily.

In the full-information setting, these two procedures have been shown to be part of a more general framework called \emph{Gradient-Based Prediction Algorithm} (\GBPA) by \citet{abernethy14, abernethy16}. In \GBPA, we give ourselves a differentiable function $\Phi:\R^d\to\R$, called a \emph{potential}, such that $\operatorname{Im}\nabla\Phi\subseteq K$. At each step $t\in[n]$, the action played by \GBPA\space is \(a_t=\nabla\Phi(-Y_{t-1})\).

Indeed, an instance of \FTRL\space with strictly convex regularizer $\psi$ is recovered by choosing $\Phi=\psi^*$. For all $\theta\in\R^d$, we then have
$\nabla\Phi(\theta)=\argmin_{a\in K} \{\sca{a}{-\theta} + \psi(a)\}$,
so the actions selected by \GBPA\space and \FTRL\space coincide.

Likewise, consider \FTPL\space with an integrable perturbation distribution $\cald$, and define
$\Phi(\theta)=\E_{\xi\sim \cald}[\phi_K(\theta+\xi)]$, where $\phi_K$ is the support function of $K$.
The potential $\Phi$ is well defined, and if $\phi_K(\theta+\xi)$ is differentiable with probability one, we can swap the expectation and gradient \citep{Bertsekas}, which yields
\begin{equation}\label{eq:gbpa1}
    \nabla\Phi(\theta)=\E_{\xi\sim\cald}[\nabla\phi_K(\theta+\xi)]=\E_{\xi\sim\cald}\big[\argmin_{a\in K}\sca{a}{-\theta-\xi}\big]\,.
\end{equation}
Therefore, the action chosen by \GBPA\space with potential $\Phi$ is the expectation of the action chosen by \FTPL\space with perturbation distribution $\cald$.

Note that for \FTRL\space, the reciprocal holds due to the properties of the Fenchel-Legendre conjugate, and every instance of \GBPA\space can be described as an instance of \FTRL. This is not the case for \FTPL, and some regularizers cannot be replicated by perturbations. Therefore, the relation between \FTPL\space and \FTRL\space is a strict inclusion in the full-information setting. \citet[Proposition 2.2]{hofbauer_sandholm} show that the function $\psi:x\mapsto-\sum_{i=1}^d\ln(x_i)$ defined on the probability simplex $\Delta^d\coloneq\{x\in\R_+^d:\|x\|_1= 1\}$ does not admit a representation of the form of Equation~\eqref{eq:gbpa1}. Generally,
characterizing which regularizers admit an equivalent perturbation is challenging, with prior work largely
focused on the simplex \citep{abernethy2015fighting, kim2019optimality}.

\subsection{Gradient-Based Prediction for Linear Bandits}

We now present \BanditsGBPA, extending \GBPA\space from the full-information setting to the linear bandit setting. In the full information setting, \GBPA\space relies on $Y_{t-1}$ to pick an action. In the bandit case, however, the learner only knows the past scalar loss $\sca{y_s}{a_s}, s<t$. It must estimate $Y_{t-1}$ through randomization, and use this estimate in the \GBPA\space step. Thus, \BanditsGBPA\space needs two additional ingredients:
\begin{enumerate}[topsep=1pt, itemsep=0pt]
    \item a \emph{sampling scheme}, which randomizes the chosen action to enable exploration. Formally, a sampling scheme is a mapping $S$ from $K$ to $\calp(K)$ the set of probability distributions on $K$, such that the chosen action $a_t$ at time $t$ will be sampled from the distribution $S(\nabla\Phi(-\hat Y_{t-1}))$.
    \item an \emph{estimation scheme}, which constructs an estimate of the unobserved loss vector from the observed scalar feedback. Formally, it is a function $E:\R\times K \times K\to \R^d$ mapping the observed scalar loss, the sampled action, and the expected action to the estimation of the loss vector, i.e., $\hat y_t=E(\sca{y_t}{a_t},a_t, \nabla\Phi(-\hat Y_{t-1}))$.
\end{enumerate}
The resulting procedure is summarized in Algorithm~\ref{alg:bandit-gbpa}. This template encompasses several existing linear bandits algorithm, including \scrible.

\begin{algorithm}[ht]
	\caption{Gradient-Based Prediction for Linear Bandits (\BanditsGBPA)}
	\label{alg:bandit-gbpa}
	\begin{algorithmic}[1]
            \REQUIRE A differentiable potential function $\Phi:\R^d\to\R$ such that $\Ima\nabla\Phi\inc K$, a sampling scheme $S:K\to\calp(K)$ and an estimation scheme $E:\R\times K \times K\to \R^d$.
            \STATE Set $\hat Y_0= 0$.
            \FOR{$t=1, \ldots, n$}
            \STATE Let $x_t=\nabla\Phi(-\hat Y_{t-1})$.
            \STATE Sample action $a_t\sim S(x_t)$ independently from the past.
            \STATE Observe punctual loss $\sca{y_t}{a_t}$.
            \STATE Compute estimator $\hat y_t=E(\sca{y_t}{a_t},a_t,x_t)$.
            \STATE Update $\hat Y_t=\hat Y_{t-1}+\hat y_t$.
            \ENDFOR
	\end{algorithmic}
\end{algorithm}

We also define a notion of unbiasedness for the sampling and estimation schemes. For $S:K\to\calp(K)$ a sampling scheme and $E:\R\times K\times K\to\R^d$ an estimation scheme, we say that $(S,E)$ is \emph{unbiased} if for all $x\in K, y\in \R^d$,
\begin{equation}
        \E_{a\sim S(x)}\left[a\right]=x
        \qquad\text{and}\qquad
        \E_{a\sim S(x)}\left[ E(\sca{y}{a},a,x)\right]=y\,.
\end{equation}
We will subsequently show that both \scrible\space and \SCFTPL\space are unbiased instances of \BanditsGBPA. 

\section{From Self-Concordant Regularization to Self-Concordant Perturbations}\label{section:from_scr_to_scp}

We contextualize \SCFTPL\space by recalling key properties of self-concordant barriers, including the local strong convexity of their dual and the controlled divergence near the boundary of their domain, reviewing the \scrible\space algorithm and its suboptimal sampling scheme, and motivating our improved sampling and estimation strategies.

\subsection{Self-Concordant Barriers}\label{section:self_concordant_barriers}

Self-concordant barriers, originally introduced for interior-point methods, are a class of functions extensively studied in \citet{nemirovskii}.
In this section, we define self-concordant barriers and present some of their key properties.
\begin{definition}[Self-Concordant Barrier]
Let $K$ be a convex set with non-empty interior and $\calr:\inter K\to \R$. The function $\calr$ is a \emph{self-concordant function} on $K$ if
\begin{itemize}[topsep=2pt, itemsep=0pt]
    \item it is three times continuously differentiable and convex,
    \item for all $(x_n)_n\in \inter K^\N$ such that $x_n$ converges to the boundary of $K$, $\calr(x_n)\to+\infty$,
    \item for all $x\in \inter K$ and $h\in \R^d$,
    $|D^3\calr(x)[h,h,h]|\le 2 (D^2\calr(x)[h,h])^{3/2}$. 
\end{itemize}
Let $\vartheta> 0$. $\calr$ is a \emph{$\vartheta$-self-concordant barrier} if it is a self-concordant function and if for all $x\in\inter K$ and $h\in\R^d$,
$|D\calr(x)[h]|\le \sqrt{\vartheta\; D^2\calr(x)[h,h]}$.

We have denoted $D^k\calr(x)[h_1,...,h_k]$ the $k$-th derivative of $\calr$ at $x$ applied to the directions $h_1,...,h_k$.
\end{definition}
For $\calr$ a $\vartheta$-self-concordant barrier on $K$ and $x\in\inter K$, the Hessian $\calr$ in $x$, $\nabla^2\calr(x)$ is positive definite. Thus, we can define the local norm with respect to the Hessian of $\calr$ in $x$
\begin{equation}
    \|y\|_{\nabla^2\calr(x)} \coloneq \sqrt{y^\top \nabla^2\calr(x)\,y}
\end{equation}
for all $y\in \R^d$. We also define the the open unit Dikin ellipsoid, which is the unit ball for this local norm:
$$W(x)\coloneq \left\{y\in \R^d:\|y-x\|_{\nabla^2\calr(x)} <1\right\}\,.$$
For all $x\in \inter K$, it holds that $W(x)\subset K$. Thus, the local geometry defined by the Hessian of $\calr$ stretches $K$ such that the Dikin ellipsoid centered in $x$, is always inside $K$.

We now introduce two important properties of self-concordant barriers that justify their use as regularizers.
First, for $\calr$ a self-concordant function, its Fenchel conjugate $\calr^*$ is also a self-concordant function and satisfies a local quadratic control of its Bregman divergence: for any $x,y\in\dom(\calr^*)$ such that
$\|y-x\|_{\nabla^2\calr^*(x)} \le \tfrac12$,
\begin{equation}\label{eq:local_sc}
B_{\calr^*}(y,x) \le \|y-x\|_{\nabla^2\calr^*(x)}^2 .
\end{equation}
This inequality can be viewed as a local strong convexity property of $\calr$,\footnote{For comparison, a convex, differentiable function $f$ is $\mu$-strongly-convex with respect to some norm $\|\cdot\|$ if and only if $B_{f^*}(x,y) \le \frac{1}{2\mu}\|x-y\|^2_*$ for all $x,y$.} and follows from standard results on self-concordant functions \citep{nemirovskii}. A detailed proof is given in Appendix~\ref{appendix:proof_lemma_local_sc}.
Second, if $\calr$ is a $\vartheta$-self-concordant barrier on a convex set $K$ with non-empty interior, we can control the rate at which the barrier diverges near the boundary of $K$. For all $x,y\in\inter K$,
\begin{equation}\label{eq:sc_barrier_growth}
\calr(y)\le \calr(x) - \vartheta \ln(1-\pi_x(y)),
\end{equation}
where $\pi_x(y)\coloneq \inf\{t>0:x+t^{-1}(y-x)\in K\}$ is the \emph{Minkowski function in $x$}.

\subsection{Self-Concordant Regularization}\label{section:SCRiBLe}

The \emph{Self-Concordant Regularization in Bandits Learning} (\scrible) algorithm, introduced by \citet{abernethy12}, 
can be described as a particular algorithm in the \BanditsGBPA\space framework.
It is defined by a $\vartheta$-self-concordant barrier $\calr$ on $K$ and a learning rate $\eta>0$.
Then, its potential is given by $\Phi=(\tfrac{1}{\eta}\calr)^*$ and given $x_t\in K$, the sampling distribution $S(x_t)$ is defined as the uniform distribution over the $2d$ poles of the Dikin ellipsoid $W(x_t)$.

The self-concordant barrier $\calr$ thus plays a dual role: it defines both the potential and the sampling scheme through its Hessian. The properties of the Dikin ellipsoid ensure that $a_t \in K$, while simultaneously providing sufficient, though suboptimal, exploration to estimate $y_t$. 

For an optimal learning rate $\eta$, \scrible\space achieves regret
$R_n=\O(d\sqrt{\vartheta n \ln n})$. Following recent works on universal self-concordant barriers \citep{lee2021universal}, we know that every convex body $K\subset\R^d$ admits a $\vartheta \le d$ barrier, yielding a $\O(d^{3/2} \sqrt{n \ln n})$ bound. For the $\ell_2$ ball, which admits a 1-self-concordant barrier, the regret improves to $\O(d \sqrt{n \ln n})$, in contrast to the hypercube, where no improvement beyond $\vartheta=n$ is possible \citep{nesterov1994interior}.
Hence, the regret of \scrible\space is suboptimal by a $\sqrt{d}$ factor. As we will discuss next, this suboptimality stems from the variance of the estimator induced by Dikin-ellipsoid sampling.

\subsection{Self-Concordant FTPL}

\SCFTPL\space uses the same potential as \scrible\space but employs a different sampling scheme. Whereas \scrible\space sampled from the poles of the Dikin ellipsoid, \SCFTPL\space uses the sampling naturally induced by \FTPL, which is inherently randomized.

We can check that this is indeed an instance of \BanditsGBPA, for the potential $\Phi=(\frac 1\eta\,\calr)^*$ where $\calr$ is the self-concordant barrier replicated by $\cald$. The unbiasedness of the sampling scheme then follows from the definition of self-concordant perturbations.

Since \SCFTPL\space shares the same potential as \scrible, its regret analysis relies on the same local strong convexity and barrier-growth properties, which we exploit in the theorem below.

\begin{theorem}\label{thm:self-conc-linear-bandits}
Let $K\subset \R^d$ be a convex body, $\eta>0$, and $\cald$ be a $\vartheta$-self-concordant perturbation for $K$, such that $\E_{\xi\sim\cald}[\nabla\phi_K(\theta+\xi)\nabla\phi_K(\theta+\xi)^\top ]$ is non-singular for all $\theta\in\R^d$. Then, \SCFTPL\space with learning rate $\eta$ and perturbation $\cald$ has its regret bounded by
\[R_n\le \frac{\vartheta \ln n}{\eta}+\eta \sum_{t=1}^n \E[\|\hat y_t\|_t^2]+2\,,\]
provided that $2\eta \|\hat y_t\|_t\le1$ almost surely for all $t\in[n]$, with $\|\cdot\|_t=\|\cdot\|_{\nabla^2\calr^*(-\eta \hat Y_{t-1})}$ where $\calr$ is the self-concordant barrier replicated by $\cald$.
\end{theorem}
This theorem is adapted from \cite{abernethy12}. We present here a proof sketch and defer the full proof to Appendix~\ref{appendix:proof_self-conc-linear-bandits}.

\begin{proofsketch}
We first establish that \SCFTPL\space is an instance of \BanditsGBPA\space with potential $\Phi=(\frac 1\eta\,\calr)^*$, where $\calr$ is the self-concordant barrier replicated by $\cald$. This allows us to use a standard regret decomposition in bandits literature, which upper-bounds the regret of \SCFTPL\space by the sum of the initial potential gap and cumulative stability terms.

The first term is $\Phi^*(u)-\min\Phi^*= \eta^{-1}(\calr(u)-\min \calr)$, which diverges as $u$ approaches the boundary of $K$. To establish a bound independent of $u$, we use a standard shrinkage argument and evaluate the regret at point $u'=(1-n^{-1}) u+n^{-1} x^*$ where we bound $\calr(u')$ by $\vartheta\ln n$ using \eqref{eq:sc_barrier_growth}.

The stability terms are $\E[B_\Phi(-\hat Y_t, -\hat Y_{t-1})]$. By the local strong convexity of self-concordant barriers \eqref{eq:local_sc}, the potential $\Phi$ is locally smooth with respect to the local norm. Provided $\eta$ is small enough, we can bound the Bregman divergence by $\eta^2 \|\hat y_t\|_t^2$, which completes the proof.
\end{proofsketch}

The same regret bound holds for the \scrible\space algorithm. Crucially, this theorem shows that the regret is controlled by two quantities: the self-concordance parameter $\vartheta$ and the sum of the variances of loss vector estimators $\hat y_t$ measured in the local norms $\|\cdot\|_t$.
Thus, improving regret reduces to designing sampling and estimation schemes that produce small $\E[\|\hat y_t\|_t^2]$.

\paragraph{Motivations for FTPL Sampling}

In \scrible, the variance of the estimator in the local norm is bounded by $d^2$, which leads to the suboptimal regret $\O(d\sqrt{\vartheta n\ln n})$. The goal of \SCFTPL\space is to reduce this variance through a different sampling and estimation scheme. More generally, the variance of the loss estimator reflects how effectively the sampling strategy explores the action set by spreading actions across directions. Without stochastic exploration, accurate estimation of the loss vector is impossible and sublinear regret cannot be achieved \citep{cesabianchilugosi}. Hence, an effective sampling scheme must exploit the available geometry of $K$ around the current point $x_t$.

The Dikin ellipsoid sampling used by \scrible\space does not do so optimally, an intuition provided by \citet{abernethy2009beating}.
When $x_t$ lies near the boundary of $K$, the Dikin ellipsoid $W(x_t)$ shrinks in order to remain centered at $x_t$ and contained in $K$. In contrast, \FTPL\space selects actions by solving a linear optimization problem over $K$, resulting in actions at extreme points of the set. This boundary-oriented sampling allows \FTPL\space sampling to exploit the full space in $K$ and, as we show later for the hypercube, can lead to lower variance of the estimator and sharper regret bounds.

\paragraph{SC-FTPL Estimation Scheme}
Let us now turn to the estimation scheme. We construct an unbiased estimator $\hat{y}_t$ for $y_t \in \mathbb{R}^d$ using only the observed bandit feedback $\langle y_t, a_t \rangle$, where $a_t \sim S(x_t)$. Following \citet{bubeck2012towards}, a natural candidate is
\begin{equation}
    \hat y_t = Q_t^{-1} a_t \langle a_t, y_t \rangle, \quad\text{with}\quad Q_t = \E_{t-1}[a_t a_t^\top],
\end{equation}
which is unbiased by construction. Its validity requires $Q_t$ to be invertible, which does not hold for general self-concordant perturbations. A counterexample is given in Appendix~\ref{appendix:non_singularity_of_Q}. To address this, the next proposition gives a sufficient condition ensuring $Q_t$ is positive definite.

\begin{proposition}\label{prop:full_support_implies_Q_invertible}
Let $K\subset \R^d$ be a convex body and $D$ be an absolutely continuous probability distribution on $\R^d$. Suppose that $\operatorname{supp}(D)=\R^d$. Then, for all $\theta\in\R^d$, the matrix 
\[Q=\E_{\xi\sim D}[\nabla\phi_K(\theta+\xi)\nabla\phi_K(\theta+\xi)^\top]\]
is positive definite.
\end{proposition}

\begin{proofsketch}
The matrix $Q$ is positive semi-definite by construction. To establish non-singularity, we show $u^\top Q u > 0$ for any $u \neq 0$. Because $D$ has full support and is absolutely continuous, this is equivalent to showing that the set 
$Z^c \coloneq \{y \in \mathbb{R}^d : \nabla \phi_K(y) \notin u^\perp\}$ has non-zero Lebesgue measure.
Geometrically, $\nabla \phi_K(y)$ is the element of $K$ that maximizes $\langle \cdot, y \rangle$. 
We construct an open convex cone $C$ such that for any $y \in C$, $\nabla \phi_K(y)$ cannot lie in the subspace $u^\perp$, because its projection onto $u$ dominates its projection onto $u^\perp$. As $C \subset Z^c$ is open and non-empty, $Z^c$ has non-zero measure, concluding the proof.
For all missing details, see Appendix~\ref{appendix:non_singularity_of_Q}.
\end{proofsketch}

In the next section, we will show that for the hypercube and the $\ell_2$ ball, one can explicitly choose self-concordant perturbations with full support, which ensures $Q_t$ remains non-singular.

\section{Analysis of SC-FTPL for Two Specific Action Sets}\label{section:specific_action_sets}

In this last section, we study \SCFTPL\space on two particular action sets: the hypercube and the unit $\ell_2$ ball. For each of them, we construct a self-concordant perturbation and detail the steps that lead to the regret bound established in Theorem~\ref{thm:regret_scftpl_hypercube} and Theorem~\ref{thm:regret_scftpl_ball}. Finally, we provide a time complexity analysis of \SCFTPL\space in both of these cases.

\subsection{Hypercube Case}

Let $K=[-1,1]^d$. We consider the \emph{entropic barrier} of \cite{bubeck_entropic}, which is defined for all convex bodies as the Fenchel conjugate of the function
\begin{equation}
    \calr^*(\theta)\coloneq \ln \int_K \exp\sca{\theta}{x}\,dx\,.
\end{equation}
We know that for all convex bodies, the entropic barrier is a $d$-self-concordant barrier \citep{chewi}. 
For the hypercube, the self-concordance parameter $\vartheta=d$ is optimal, as there does not exist any $\vartheta$-self-concordant barrier with $\vartheta < d$ \citep[Proposition 2.3.6]{nesterov1994interior}.

For the hypercube, while the entropic barrier $\calr$ has no closed-form expression, its conjugate $\calr^*$ and its conjugate gradient are
\begin{equation*}
\calr^*(\theta)=\sum_{i=1}^d \ln \frac{2\sinh \theta_i}{\theta_i}
,\quad\text\quad
\nabla\calr^*(\theta)=[L(\theta_i)]_{i=1}^d
\end{equation*}
for all $\theta\in\R^d$, where $L:t\mapsto \coth(t)-1/t$.
\begin{proof}
The exponential of the entropic barrier's conjugate is
\begin{equation*}
\exp\calr^*(\theta) 
= \int_{[-1,1]^d}e^{\sca{\theta}{x}}dx
= \prod_{i=1}^d \int_{-1}^1 e^{\theta_ix_i} dx_i
= \prod_{i=1}^d \frac{e^{\theta_i}-e^{-\theta_i}}{\theta_i}\,.
\end{equation*}
Taking the logarithm yields the stated expression. For the gradient, notice that $\calr^*$ is additively separable and that the derivative of $t\mapsto \ln(2\sinh (t)/t)=\ln 2 + \ln\sinh (t)-\ln t$ is $L$.
\end{proof}
Hence, $\nabla\calr^*$ is component-wise separable.
Similarly, the gradient of the support function is $\nabla\phi_K:\theta \mapsto[\sgn(\theta_i)]_{i=1}^d$. This separability allows us to construct easily a perturbation that replicates $\calr$ on $[-1,1]^d$.

\begin{proposition}\label{prop:scperturbation_hypercube_existence}
There exists a probability distribution over $\R^d$ whose marginal distributions are independent and have probability density function
\begin{equation*}
    f:t\in\R\mapsto \frac{1}{2t^2}-\frac{1}{2\sinh^2(t)}\,.
\end{equation*}
This is a $d$-self-concordant perturbation for $[-1,1]^d$.
\end{proposition}
The proof of this proposition is quite straightforward and is deferred to Appendix~\ref{appendix:scperturbation_hypercube_existence}. Now that we have exhibited a self-concordant perturbation for $[-1,1]^d$, we upper-bound the variance of the local norm of the estimators of \SCFTPL, in order to obtain an upper bound on the regret.

\begin{proposition}\label{prop:scperturbation_hypercube_bound}
Let $\cald$ be the self-concordant perturbation defined in Proposition~\ref{prop:scperturbation_hypercube_existence} and $\calr$ be the entropic barrier on the hypercube. Consider the \SCFTPL\space algorithm run with perturbation $\cald$ and any learning rate $\eta>0$. Then, for all $t\in[n]$, the estimator $\hat y_t$ satisfies
\begin{equation*}
    \E_{t-1}[\|\hat y_t\|^2_t] \le \frac{1}{3}d
    \qquad\text{and}\qquad
    \|\hat y_t\|^2_t\le 3\,d\;\;\text{a.s.,}
\end{equation*}
where $\|\cdot\|_t=\|\cdot\|_{\nabla^2\calr^*(-\eta\hat Y_{t-1})}$.
\end{proposition}
The proof is mainly computational and is provided in Appendix~\ref{appendix:scperturbation_hypercube_bound}. Crucially, the bound of the variance of the local norm of the estimator improves by a $3d$ factor over the \scrible's $d^2$ bound. This confirms that, on the hypercube, \SCFTPL\space achieves more accurate loss vector estimation.
Now that we bounded the variance of the local norm of the estimator, we can establish Theorem~\ref{thm:regret_scftpl_hypercube}.

\paragraph{Proof sketch of Theorem \ref{thm:regret_scftpl_hypercube}.}
The second inequality of Proposition~\ref{prop:scperturbation_hypercube_bound} and the assumption on $n$ and $d$ ensure $\eta\|\hat y_t\|^2 \le 1/2$ almost surely. Theorem~\ref{thm:self-conc-linear-bandits} then applies, yielding a regret bound in terms of $\vartheta=d$ and $\E[\|\hat y_t\|_t^2]\le d/3$. Combining these and optimizing $\eta$ gives the stated regret bound. Full details are in Appendix~\ref{appendix:regret_scftpl_hypercube}.
\hfill$\blacksquare$

This theorem bounds the regret of \SCFTPL\space on the hypercube in $\O(d\sqrt{n \ln n})$ for an optimal choice of $\eta$. This bound improves over \scrible\space by a factor of $\sqrt{d}$ and matches the lower bound established by \cite{dani_pricebanditinfo} up to logarithmic factors.
Finally, we provide a complexity analysis of \SCFTPL\space on the hypercube in Appendix \ref{appendix:complexity_analysis_hypercube}, and prove that it has a per-round complexity in $\O(d)$. This matches the complexity of \scrible\space and \textsc{Osmd} in the same setting.

\subsection{\texorpdfstring{$\ell_2$}{l2} Ball Case}

Let us now consider the case where the action set is $\B^d$ the unit $\ell_2$ ball. We consider the \emph{log-barrier} on the ball given by
\[\calr (x)\coloneq -\ln (1-\|x\|_2^2)\,,\quad x\in\inter\B^d\,. \]
This is a 1-self-concordant barrier on $\B^d$. Here, both $\calr$ and $\phi_{\B^d}$ are spherically invariant, i.e., depend only on the norm of their argument. Exploiting this symmetry, we construct an explicit perturbation replicating $\calr$ on $\B^d$:

\begin{proposition}\label{prop:scperturbation_ball_existence}
Let $\mathbf T$ be a random vector in $\R^d$ following a multivariate $t$-distribution with location 0, scale matrix $I_d$ and $d+1$ degrees of freedom. Let $U$ be a random variable following a uniform distribution on $[0,1]$, independent from $\mathbf T$. Then, the distribution of
\[\xi = \frac{\mathbf T}{\sqrt{d+1}\ U}\]
is a 1-self-concordant perturbation for $\B^d$.
\end{proposition}
While the computations leading to the proof of Proposition~\ref{prop:scperturbation_ball_existence} reduce to a one-dimensional radial calculation, identifying the perturbation that exactly reproduces the barrier geometry is non trivial and, to the best of our knowledge, novel. These computations are detailed in Appendix~\ref{appendix:ball}. Now that we have exhibited a self-concordant perturbation for $\B^d$, we derive bounds on the estimators of \SCFTPL\space for this choice of perturbation.

\begin{proposition}\label{prop:scperturbation_ball_bound}
Let $\cald$ be the $1$-self-concordant perturbation defined in Proposition~\ref{prop:scperturbation_ball_existence} and $\calr$ be the log-barrier on $\B^d$. Consider Algorithm~\ref{alg:scftpl} run with perturbation distribution $\cald$ and any learning rate $\eta>0$. Then, for all $t\in[n]$, the estimator $\hat y_t$ satisfies the following inequalities
\begin{enumerate}
    \item $\E_{t-1}[||\hat y_t||_t^2]\le \frac{5}{4}d^2$,
    \item $||\hat y_t||_t^2 \le d^2\eta \,\|\hat Y_{t-1}\|_2 + 4\,d^2$ almost surely, and
    \item $\E_{t-1}[||\hat y_t||^2_2]\le d^2\eta\, \|\hat Y_{t-1}\|_2 + 2\,d^2$.
\end{enumerate}
where $\|\cdot\|_t=\|\cdot\|_{\nabla^2\calr^*(-\eta\hat Y_{t-1})}$.
\end{proposition}

\begin{proofsketch}
Fix $t\in[n]$ and let $\theta=-\eta\hat Y_{t-1}$. By the spherical symmetry of $\cald$, $Q_t$ is invariant under rotations that fix $\theta$, which implies that it admits the decomposition $Q_t = q\,P_\theta + q_\perp\,P_{\theta^\perp}$,
where $P_\theta$ and $P_{\theta^\perp}$ are the orthogonal projection matrices onto $\theta$ and $\theta^\perp$, respectively. We then show, using properties of the multivariate $t$-distribution, that both coefficients are positive and satisfy
$q \ge 1/d$ and $ q_\perp \ge 1/(d(\|\theta\|+1))$.
Thus $Q_t^{-1}=q^{-1}P_\theta+q_\perp^{-1}P_{\theta^\perp}$.

Substituting $Q_t^{-1}$ and $\nabla^2\calr^*(\theta)$ into the expressions of
$\E_{t-1}[\|\hat y_t\|_t^2]$, $\|\hat y_t\|_t^2$, and $\E_{t-1}[\|\hat y_t\|_2^2]$, each quantity reduces to a one–dimensional inequality in $\|\theta\|$. Real analysis computations then yields the three stated bounds.
The full proof, including the derivation of the bounds on $q$ and $q_\perp$, is deferred to Appendix~\ref{appendix:scperturbation_ball_bound}.
\end{proofsketch}

From these bounds, we derive Theorem~\ref{thm:regret_scftpl_ball}, which bounds the regret of \SCFTPL\space on the $\ell_2$ ball.

\paragraph{Proof sketch of Theorem \ref{thm:regret_scftpl_ball}.}
The main difficulty is that, unlike the hypercube case, there is no uniform almost-sure bound on $\|\hat y_t\|_t$ independent of $\hat Y_{t-1}$. As a result, the regret decomposition of Theorem~\ref{thm:self-conc-linear-bandits}, which requires $\eta\|\hat y_t\|_t\le 1/2$ for all $t$, cannot be applied directly.

To address this issue, we work on the high-probability event
\begin{equation*}
    E\coloneq\{\forall t\in[n],\, \eta\|\hat y_t\|_t\le 1/2\}\,.
\end{equation*}
Using the almost-sure bound of Proposition~\ref{prop:scperturbation_ball_bound}, we show that $E$ holds whenever $\sup_{t<n}\|\hat Y_t\|^2$ remains below a threshold $1/C_n$, where $C_n=o(n^{-5/2})$ is an explicit quantity arising from the analysis and for which the resulting failure probability is negligible compared to the main regret term.

To control the growth of $\|\hat Y_t\|^2$, we use the $\ell_2$-norm variance bound of Proposition~\ref{prop:scperturbation_ball_bound} to construct an appropriate supermartingale. Applying Ville’s inequality yields a tail bound of order $C_n n^2$ on the probability that $\sup_{t<n}\|\hat Y_t\|^2>1/C_n$.

We then define the stopping time
\begin{equation*}
    \tau\coloneq\inf\{t<n:\|\hat Y_t\|^2>1/C_n\}\wedge n\,,
\end{equation*}
and decompose the regret at time $\tau$. On the event $\{\tau=n\}$, the condition $\eta\|\hat y_t\|_t\le 1/2$ holds throughout, and an optional stopping argument allows us to apply the same regret analysis as in Theorem~\ref{thm:self-conc-linear-bandits}. Using $\vartheta=1$ and the local norm variance bound from Proposition~\ref{prop:scperturbation_ball_bound}, this yields the leading term $d\sqrt{5n\ln n}+2$ for an optimal value of the learning rate $\eta$.

On the complementary event $\{\tau<n\}$, the regret is trivially bounded by $2n$, which multiplied by the failure probability contributes a lower-order term $\O(C_n n^3)=o(\sqrt{n})$. Combining both parts concludes the proof. Full details are given in Appendix~\ref{appendix:regret_scftpl_ball}.
\hfill$\blacksquare$

Compared to the hypercube setting, our results on the $\ell_2$ ball are weaker in two aspects. First, we were unable to establish an $\O(d)$ bound on the variance of the local norm of $\hat y_t$. Our analysis yields a quadratic dependence on $d$, matching the \scrible\space estimator, which leads to the $\O(d\sqrt{n \ln n})$ regret in Theorem~\ref{thm:regret_scftpl_ball}. Second, there is no uniform almost-sure bound on $\|\hat y_t\|_t$ independent of $\hat Y_{t-1}$. This justify the need for the Ville's inequality and optional stopping time arguments we developed in the proof of Theorem~\ref{thm:regret_scftpl_ball}. These arguments introduce an additional term in the regret bound.
While this additional term is asymptotically negligible when compared to the main $\O(d\sqrt{n\ln n})$ term, it may still be important for small values of $d$.

We include a discussion of the computational cost of \SCFTPL\space on the $\ell_2$-ball in Appendix~\ref{appendix:complexity_analysis_ball}. We show that it has a per-round complexity in $\mathcal{O}(d)$: both sampling the perturbation $\xi$ as defined in Proposition~\ref{prop:scperturbation_ball_existence} and computing $a_t$ are $\mathcal O(d)$. Then, computing the loss estimator $\hat y_t$ can be done in $\mathcal O(d)$ using numerical integration and leveraging the structure of $Q_t^{-1}$. This per-round computational cost matches the complexity of \scrible\space and of \cite{hoeven2018many}'s algorithm on the $\ell_2$-ball.

\section{Future Directions}\label{section:conclusion}

Our work open up to several unexplored research directions. First, we have shown that the regret of \SCFTPL\space on the $\ell_2$ ball is bounded in $\O(d\sqrt{n \ln n})$ for our choice of self-concordant perturbation. Whether this rate is tight for \FTPL\space introduced by any perturbation scheme remains an open question.

Second, we have not established the existence of self-concordant perturbations for general convex bodies. It remains unclear whether such perturbations exist beyond specific domains. A promising direction is to investigate the entropic barrier, a universal self-concordant barrier defined for any convex body \citep{bubeck_entropic}, and determine whether it systematically admits a corresponding perturbation distribution.

Finally, \citet{HazanLevy} showed how self-concordant regularization can be used beyond linear bandits, when the loss functions are convex. It would be natural to explore whether the perturbation-based perspective developed here can yield similar benefits in that setting.

\acks{L.\@ L{\'e}vy, J.-L.\@ Valeau, A.\@ Akhavan, and P.\@ Rebeschini were funded by UK Research and Innovation (UKRI) under the UK government’s Horizon Europe funding guarantee (grant number EP/Y028333/1).}

\bibliography{refs}

\appendix

\newpage
\section{Notation}

In this section, we collect the main pieces of notation used in this paper.

\paragraph{Convex Analysis Notation}
Let $K\subset \R^d$ be a convex body, i.e., a non-empty compact convex set. We denote $\phi_K:\theta\mapsto\max_{x\in K}\sca{x}{\theta}$ its \emph{support function} and for all $x\in\R^d$, $\pi_x:y\in\R^d\mapsto \inf\{t>0:x+t^{-1}(y-x)\in K\}$ its \emph{Minkowski function in $x$}.

Let $f$ be a closed, proper convex function on $\R^d$. Its \emph{Fenchel conjugate} is \(f^*:\theta\in\R^d\mapsto \sup_{x\in\R^d}\sca{x}{\theta}-f(x)\).
We refer to \citet{rockafellar} for an extensive treatment of support functions and Fenchel duality.
If $f$ is differentiable, we define the \emph{Bregman divergence with respect to $f$} for $x,y\in\dom f$ as $B_f(x,y)=f(x)-f(y)-\sca{\nabla f(y)}{x-y}$.

\paragraph{Bandit Notation}
We denote $n\in\N$ the horizon and $K$ the action set. For all round $t\in[n]\coloneq\{1,\ldots,n\}$, the loss vector is denoted $y_t$ and the learner's action $a_t$.

Let $(\mathcal F_t)_{0\le t\le n}$ be the natural filtration defined by $\mathcal F_t= \sigma(a_s,s\le t)$, and let $\E_t[\cdot] = \E[\cdot|\mathcal F_t]$ be the conditional expectation given the first $t$ rounds.

For \SCFTPL, $\cald$ denotes the $\vartheta$-self-concordant perturbation and $\eta$ the learning rate. At round $t$, $\xi_t$ is the perturbation sampled from $\cald$, $Q_t=\E_{t-1}[a_ta_t^\top]$, $\hat y_t=Q_t^{-1}a_t\sca{y_t}{a_t}$ is the loss estimator, and $\hat Y_t=\sum_{s=1}^t \hat y_s$ is the cumulated estimate.

Finally, $\calr$ denotes the self-concordant barrier replicated by $\cald$, and 
\[\|\cdot\|_t=\|\cdot\|_{\nabla^2\calr^*(-\eta \hat Y_{t-1})}\]
is the local norm induced by the Hessian of $\calr^*$ in $-\eta\hat Y_{t-1}$.

\section{Heavy-Tailed Nature of Self-Concordant Perturbations}\label{appendix:heavy_tailed}

In this section, we highlight an interesting property of self-concordant perturbations that is not needed for the main proofs but provides valuable intuition. Specifically, these perturbations are heavy-tailed, which helps explain why they naturally induce strong exploration in bandit algorithms. Understanding this behavior sheds light on the role of the perturbation in \SCFTPL\space and related algorithms.

\begin{proposition}\label{prop:heavy_tail}
    Let $K\subset \R^d$ be a convex body and $\cald$ be a self-concordant perturbation for $K$. Then $\cald$ does not have a finite first moment.
\end{proposition}
The proof of Proposition \ref{prop:heavy_tail} makes use of the following lemma from \citet{Bertsekas} that justifies exchanging the gradient and expectation operators.

\begin{lemma}[\citealp{Bertsekas}]\label{lemma:bertsekas}
    Let $h:\R^d\to \R$ be a convex function and $\xi$ be a random vector on $\R^d$, whose distribution is absolutely continuous w.r.t the Lebesgue measure. Assume that for all $x\in\R^d$, $\E|h(x-\xi)|<\infty$ and define $H(x)=\E[h(x-\xi)]$. Then $h$ is differentiable almost everywhere, $H$ is differentiable everywhere and for all $x\in\R^d$,
    \begin{equation*}
        \nabla H(x)=\E[\nabla h(x-\xi)]\,.
    \end{equation*}
\end{lemma}

\begin{proof}\textbf{of Proposition \ref{prop:heavy_tail}.}
We proceed by contradiction. Assume that $\E\|\xi\|<+\infty$ with $\xi\sim\cald$. For all $\theta\in\R^d$,
\begin{equation*}
    \E|\phi_K(\theta+\xi)|\le \big(\|\theta\|+\E\|\xi\|\big)\sup_{x\in K}\|x\|<+\infty\,.
\end{equation*}
From Lemma \ref{lemma:bertsekas}, it follows that for all $\theta\in\R^d$,
\begin{equation*}\nabla\E[\phi_K(\theta+\xi)]=\E[\nabla\phi_K(\theta+\xi)]=\nabla\calr^*(\theta)\,.
\end{equation*}
So there exists $C\in\R$ such that
\begin{equation*}
    \calr^*(\theta)=\E[\phi_K(\theta+\xi)]+C\quad\text{for all $\theta\in\R^d$.}
\end{equation*}
The support function $\phi_K$ is $L$-Lipschitz with $L=\sup_{x\in K}\|x\|$, so for all $\theta,\xi\in\R^d$,
\begin{equation*}
    \phi_K(\theta+\xi)-\phi_K(\theta)\ge -L\|\xi\|\,.
\end{equation*}
Taking the expectation over $\xi\sim \cald$ yields
\begin{equation}\label{eq:13}
    \calr^*(\theta)-\phi_K(\theta)\ge -L\,\E\|\xi\|+C\,,\quad\forall\theta\in\R^d\,.
\end{equation}
For all $x\in K$, we have
\setlength{\jot}{3pt}
\begin{align*}
    \calr^*(\nabla\calr(x))&=\sca{\nabla\calr(x)}{x}-\calr(x)\\
    &\le \phi_K(\nabla\calr(x))-\calr(x)\,.
\end{align*}
Hence $(\calr^*-\phi_K)(\nabla\calr(x))$ goes to $-\infty$ when $\calr(x)$ goes to $+\infty$.
Because $\calr$ is unbounded from above on $K$, this directly contradicts (\ref{eq:13}). So $\cald$ does not have a finite first moment.
\end{proof}
This property highlights an important distinction between defining self-concordant perturbations at the level of gradients versus potentials. In \cite{abernethy14}, the \emph{replication} property between a perturbation distribution $\cald$ and a regularizer $\psi$ is defined at the potential level, i.e.,
\begin{equation}\label{eq:replication_potential}
    \psi^*(\theta)= \E_{\xi\sim\cald}[\phi_K(\theta+\xi)],\quad \theta\in\R^d.
\end{equation}
By exchanging the gradient and expectation, they then derive the equality of the action in expectation,
\begin{equation}\label{eq:replication_gradient}
    \nabla\psi^*(\theta)= \E_{\xi\sim\cald}[\nabla\phi_K(\theta+\xi)],\quad \theta\in\R^d.
\end{equation}
This second equality can be called \emph{replication property at the gradient level}. In Definition \ref{def:scp} however, we define self-concordant perturbations using the replication at the gradient level. Proposition \ref{prop:heavy_tail} explains this: since $\cald$ lacks a finite first moment, the expectation $\E[\phi_K(\theta+\xi)]$ may fail to exist. However, because $K$ is bounded, the expectation $\E[\nabla\phi_K(\theta+\xi)]$ is always well-defined.

\section{Proof of Self-Concordance Results}

\subsection{Proof of Equation (\ref{eq:local_sc})}\label{appendix:proof_lemma_local_sc}

\begin{proof}
From \cite[2.2]{nemirovskii}, we know that if $\calr$ is a self-concordant function, then so is $\calr^*$ and that for all $x\in\inter K$, $y\in W(x)$, it holds that
\begin{equation}\label{eq:nemirovskii_221}
    \calr(x+y) \le \calr(x)+\nabla\calr(x)^\top y+\rho(\|x-y\|_x)
\end{equation}
where $\rho(t)=-\ln(1-t)-t$. Because it is also a self-concordant function, (\ref{eq:nemirovskii_221}) holds for $\calr^*$, and we have that if $\|y-x\|_{\nabla^2\calr^*(x)}<1$, then
$$\calr^*(x+y) \le \calr^*(x)+\nabla\calr^*(x)^\top y+\rho(\|x-y\|_{\nabla^2\calr^*(x)}).$$
Thus, all that remains is to prove that $\rho(t)\le t^2$ for $t\le 1/2$. Let $f(t)=t^2-\rho(t)$. Then, $f'(t)=2t-\frac{1}{1-t}+1=\frac{t(2t-1)}{t-1}$. For all $t\le 1/2$, we have that $f'(t)\ge 0$, hence $f(t)\ge f(0)=0$. This proves that that $\rho(t)\le t^2$ for all $t\le 1/2$, which concludes the proof.
\end{proof}

\subsection{Proof of Theorem \ref{thm:self-conc-linear-bandits}}\label{appendix:proof_self-conc-linear-bandits}

The proof of Theorem \ref{thm:self-conc-linear-bandits} makes use of the following lemma, which bounds the regret of \BanditsGBPA\space when the sampling and estimation schemes are unbiased, following a standard regret decomposition in the bandit literature. For instance, a similar, though not identical, statement can be found in \citet[Section 5]{bubeck_cesabianchi_12}.

\begin{lemma}\label{lemma:bandits_gbpa}
    Let $K$ be a convex set. Let $\Phi:\R^d\to\R$ be convex, differentiable and such that $\Ima\nabla\Phi\subseteq K$.
    Then, Algorithm \ref{alg:bandit-gbpa} with potential $\Phi$ and unbiased sampling and estimation schemes satisfies, for all $u\in K$,
    \begin{equation*}
        R_n(u)\le \Phi^*(u)-\min_{K}\Phi^*+\sum_{t=1}^n \E\big[B_\Phi(-\hat Y_t, -\hat Y_{t-1})\big]\,.
    \end{equation*}
\end{lemma}

\begin{proof}\textbf{of Lemma \ref{lemma:bandits_gbpa}.}
The first step of the proof is to notice that, because of the unbiasedness of sampling and estimation schemes, it holds that for all $t\in[n]$,
\begin{equation*}
    \E_{t-1}[\sca{y_t}{a_t-u}]=\sca{y_t}{x_t-u} = \E_{t-1}[\sca{\hat y_t}{x_t-u}]\,.
\end{equation*}
It follows that the regret can be written as
\begin{align*}
R_n(u)&=\E\Big[\sum_{t=1}^n \sca{\hat y_t}{x_t-u}\Big]\\
&=\sum_{t=1}^n \E\big[\sca{\hat Y_{t}-\hat Y_{t-1}}{\nabla\Phi(-\hat Y_{t-1})}\big]-\sum_{t=1}^n\E\big[\sca{\hat y_t}{u}\big]\\
&=\sum_{t=1}^n\E\big[ -\Phi(-\hat Y_t)+\Phi(-\hat Y_{t-1})+B_\Phi(-\hat Y_t,-\hat Y_{t-1})\big]-\E\big[\sca{\hat Y_n}{u}\big]\,,
\end{align*}
where the second equality follows from the definition of $x_t$ and the third equality follows from the definition of the Bregman divergence. By telescoping the $\Phi(-\hat Y_t)$ terms, we obtain
\begin{equation}\label{eq:proof_lemma6_1}
R_n(u)=\E\big[-\Phi(-\hat Y_n)+\Phi(-\hat Y_0)\big] + \sum_{t=1}^n \E\big[B_\Phi(-\hat Y_t, -\hat Y_{t-1})\big]-\E\big[\sca{\hat Y_n}{u}\big]\,.
\end{equation}
Because $\Phi$ is proper, closed, and convex, we have $\Phi=\Phi^{**}$. Therefore,
\[\Phi(-\hat Y_0)=\Phi(0)=\sup_{x \in K}(-\Phi^*(x))=\min_{x \in K}\Phi^*(x)\,,\]
and
\[\Phi(-\hat Y_n)=\sup_{a\in K}\big\{\sca{a}{-\hat Y_n}-\Phi^*(a)\big\}\ge \sca{u}{-\hat Y_n}-\Phi^*(u)\,.\]
It follows that
\begin{equation*}
    -\Phi(-\hat Y_n)-\sca{u}{\hat Y_n}\le \Phi^*(u)\quad\text{almost surely,}
\end{equation*}
Combining this last inequality into (\ref{eq:proof_lemma6_1}) concludes the proof.
\end{proof}

\begin{proof}\textbf{of Theorem \ref{thm:self-conc-linear-bandits}.}
We first prove that Self-Concordant FTPL is an unbiased instance of Bandits-GBPA. Its potential is given by $\Phi=(\frac{1}{\eta}\calr)^*$, where $\calr$ denotes the self-concordant barrier replicated by $\cald$. By properties of the Fenchel conjugate, we have
\begin{equation*}
    \Phi(\theta)=\frac{1}{\eta}\calr^*(\eta\,\theta)\quad\text{and}\quad
    \Phi^*(x)=\frac{1}{\eta}\calr(x)
\end{equation*}
Hence, for all $t\in[d]$,
\setlength{\jot}{3pt}
\begin{align*}
    \E_{t-1}[a_t]
    &=\E_{\xi_t}\big[\nabla\phi_K(\xi_t-\eta\,\hat Y_{t-1})\big]\\
    &=\nabla\calr^*(-\eta\,\hat Y_{t-1})\\
    &=\nabla\Phi(-\hat Y_{t-1})\,,
\end{align*}
where the second equality follows from the definition of self-concordant perturbations. This shows that the sampling scheme is unbiased. Moreover,
\begin{equation*}
    \E_{t-1}[\hat y_t]=Q_t^{-1}\cdot\E_{t-1}\big[a_ta_t^\top\big]y_t=y_t\,.
\end{equation*}
Thus, Self-Concordant FTPL satisfies the unbiasedness condition of Bandits-GBPA, and Lemma~\ref{lemma:bandits_gbpa} yields the regret decomposition
\begin{equation}\label{eq:proof_thm7_0}
    R_n(u)\le \Phi^*(u)-\min_{x \in K}\Phi^*(x) + \sum_{t=1}^n \E\big[B_\Phi(-\hat Y_t,-\hat Y_{t-1})\big]\,,
\end{equation}
for all $u\in K$. Expressing (\ref{eq:proof_thm7_0}) in terms of $\calr$, we obtain
\begin{equation}\label{eq:proof_thm7_1}
    R_n(u)\le \frac{\calr(u)-\min_{x\in K}\calr(x)}{\eta}+\frac{1}{\eta}\sum_{t=1}^n \E\big[B_{\calr^*}(-\eta \hat Y_{t},-\eta\hat Y_{t-1})\big]\,.
\end{equation}
Let $x^*=\argmin_{x \in K}\calr(x)$ and $\kappa\in(0,1)$, and define $u'=\kappa x^*+(1-\kappa){u}$. Then $\pi_{x^*}(u')=(1-\kappa)\pi_{x^*}(u)\le 1-\kappa$. By Equation (\ref{eq:sc_barrier_growth}), it follows that
\[\calr(u')-\calr(x^*)\le -\vartheta\ln \kappa\,.\]
We now decompose the regret as
\begin{align}
R_n(u)
&=\E\Big[\sum_{t=1}^n\sca{y_t}{a_t-u'+u'-u}\Big]\notag\\
&=R_n(u')+\sum_{t=1}^n\sca{ y_t}{u'-u}\notag\\
&\le \frac{\calr (u')-\calr(x^*)}{\eta}+\frac{1}{\eta}\sum_{t=1}^n \E[B_{\calr^*}(-\eta \hat Y_{t},-\eta\hat Y_{t-1})]+\kappa\sum_{t=1}^n\sca{y_t}{x^*-u}\notag\\
&\le -\frac{1}{\eta}\vartheta \ln \kappa +\frac{1}{\eta}\sum_{t=1}^n \E[B_{\calr^*}(-\eta \hat Y_{t},-\eta\hat Y_{t-1})] + 2n\kappa\,,\label{eq:proof_thm7_15}
\end{align}
where the first inequality uses (\ref{eq:proof_thm7_1}) in $u'$.
Since (\ref{eq:proof_thm7_15}) holds for any $\kappa\in(0,1)$, taking $\kappa=\frac{1}{n}$ gives
\begin{equation}\label{eq:proof_thm7_2}
    R_n(u)\le \frac{1}{\eta}\vartheta \ln n +\frac{1}{\eta}\sum_{t=1}^n \E[B_{\calr^*}(-\eta \hat Y_{t},-\eta\hat Y_{t-1})] + 2\,.
\end{equation}
It remains to bound the Bregman divergence term.
Since $\|\eta \hat y_t\|_t \le 1/2$ almost surely for all $t \in [n]$, 
we can apply the local smoothness inequality of (\ref{eq:local_sc}), which yields
\begin{align*}
    R_n(u)&\le \frac{1}{\eta}\vartheta \ln n +\frac{1}{\eta}\sum_{t=1}^n \E[\|-\eta \hat Y_{t}+\eta\hat Y_{t-1}\|_{\nabla^2\calr^*(-\eta\hat Y_{t-1})}^2] + 2\\
    &= \frac{1}{\eta}\vartheta \ln n +\eta\sum_{t=1}^n \E[\|\hat y_t\|_{t}^2] + 2\,.
\end{align*}
\end{proof}

\section{Estimation Scheme of \textsc{Sc-Ftpl}}\label{appendix:non_singularity_of_Q}

\subsection{Proof of Proposition \ref{prop:full_support_implies_Q_invertible}}
\begin{proof}
Let $\theta\in\R^d$. The matrix $Q$ is symmetric and positive semi-definite by construction. To prove non-singularity, it suffices to show that $Q$ is positive definite, i.e., $u^\top Qu>0$ for all $u\in \R^d\backslash\{0\}$. Consider an arbitrary non-zero vector $u$. We have that
$$u^\top Qu=\E_{\xi\sim D}[(u^\top\nabla\phi_K(\theta+\xi))^2]\,.$$
Since the integrand is non-negative, the expectation is zero if and only if $u^\top \nabla \phi_K(\theta + \xi) = 0$ almost surely with respect to $D$. Since $D$ is absolutely continuous with full support on $\mathbb{R}^d$, this condition is equivalent to the set $Z = \{ y \in \mathbb{R}^d : \nabla \phi_K(y) \in u^\perp \}$ having full Lebesgue measure in $\mathbb{R}^d$. Thus, in order to prove $u^\top Qu>0$, it suffices to show that $\R^d\,\backslash\,Z$ has non-zero Lebesgue measure.

Since $K$ is bounded and has non-empty interior, there exists $x \in K$ and $\varepsilon, M > 0$ such that $B(x, \varepsilon) \subset K \subset B(0, M)$. Without loss of generality, we may assume $\langle u, x \rangle \ge 0$. Define the cone
$$C=\left\{y \in \R^d: \sca{y}{\frac{u}{\|u\|}} > \frac{2M}{\varepsilon}\|p_{u^\perp}(y)\|\right\}\,,$$
where $p_{u^\perp}$ denotes the orthogonal projection onto $u^\perp$. The set $C$ is open and non-empty, and thus has non-zero Lebesgue measure. We now show that $C\cap Z=\emptyset$.

Let $y\in C$. Consider the point $z=x+\varepsilon\frac{u}{\|u\|} \in K$. It holds that
\begin{align*}
    \sca{y}{z}&= \sca{p_u(y)}{p_u(x)} + \sca{p_{u^\perp}(y)}{p_{u^\perp}(x)}+\varepsilon \sca{y}{\frac{u}{\|u\|}}\\
    &> -\|p_{u^\perp}(y)\|\,\|p_{u^\perp}(x)\| + 2M \|p_{u^\perp}(y)\|\\
    &\ge M \,\|p_{u^\perp}(y)\|
\end{align*}
Moreover, for any $h\in K\cap u^\perp$, observe that
$$\sca{y}{h}=\sca{p_{u^\perp}(y)}{h}\le M \,\|p_{u^\perp}(y)\|\,.$$
Combining these results, we strictly have $\langle z, y \rangle > \sup_{h \in K \cap u^\perp} \langle h, y \rangle$. Since 
$$\nabla \phi_K(y) \in \arg\max_{w \in K}\langle w, y \rangle,$$
it follows that $\nabla \phi_K(y) \notin u^\perp$. Therefore, $\langle \nabla \phi_K(y), u \rangle \neq 0$ for all $y \in C$, implying $C \subseteq \R^d \,\backslash \,Z$. Since $C$ has non-zero measure, $u^\top Q u > 0$.
\end{proof}

\subsection{Ill-Definition of the Estimator}

The \SCFTPL\space estimator estimates the loss vector $y_t$ using
\begin{equation*}
    \hat y_t=Q_t^{-1}a_t\sca{a_t}{y_t}
    \quad\text{with }Q_t=\E_{t-1}[a_ta_t^\top ]\,.
\end{equation*}
For this estimator to be well-defined, the covariance matrix $Q_t$ must be non-singular. We established in Proposition \ref{prop:full_support_implies_Q_invertible} that if the perturbation distribution $\mathcal{D}$ has full support, then $Q_t$ is invertible. Furthermore, in Section \ref{section:specific_action_sets}, when $K$ is the hypercube or the $\ell_2$ ball, we explicitly constructed self-concordant perturbations satisfying this condition.

A natural theoretical question is whether the self-concordance property alone is sufficient to guarantee the non-singularity of $Q_t$. In this section, we show that the answer is negative. We provide a counterexample of a valid self-concordant perturbation that results in a singular covariance matrix.

Consider the hypercube $K=[-1,1]^d$ with $d\ge 2$ and the univariate density function
$$f(t)= \frac{1}{2t^2}-\frac{1}{2\sinh^2t}\,,\quad t\in\R\,.$$
Proposition \ref{prop:scperturbation_hypercube_existence} established that if a random vector $\xi$ has density $p(x) = \prod_{i=1}^d f(x_i)$,
i.e., if the coordinates of $\xi$ are independent and have density $f$, then the distribution $\cald$ of $\xi$ is a $d$-self-concordant perturbation for $K$. Since $\mathcal{D}$ has full support on $\mathbb{R}^d$, the resulting covariance matrix $Q_t$ is always non-singular.

However, the self-concordance property for the hypercube relies solely on the marginal distributions of $\xi$, not on the independence of its coordinates. We exploit this to construct a degenerate perturbation. Let $\mathcal{D}'$ be an absolutely continuous distribution on $\mathbb{R}^d$ defined by the density:
$$q(x) = 2^{d-1} \cdot \mathbf{1}\left\{x \in (\mathbb{R}_+)^d \cup (\mathbb{R}_-)^d\right\} \cdot \prod_{i=1}^d f(x_i)\,.$$
Because the function $f$ is even, one can verify that the marginal distributions of $\cald'$ are identical to those of $\mathcal{D}$, each coordinate has density $f$.  Consequently, $\mathcal{D}'$ remains a valid $d$-self-concordant perturbation for $K$.

Consider the first round of \SCFTPL\space ($t=1$) with $\hat{Y}_0 = 0$ and perturbation $\xi_1 \sim \mathcal{D}'$. 
The action is given by:
$$a_1 = \argmin_{a \in K} \langle a, -\xi_1 \rangle = \left( \operatorname{sgn}(\xi_{1,i}) \right)_{1\le i\le d}.$$
By construction of $\mathcal{D}'$, the vector $\xi_1$ lies almost surely in $(\mathbb{R}_+)^d \cup (\mathbb{R}_-)$. 
Therefore, $a_1$ takes values in $\{\mathbf{1}_d, -\mathbf{1}_d\}$ almost surely. In both cases, the outer product is the full matrix $a_1 a_1^\top = \mathbf{1}\mathbf{1}^\top$. 
Taking the expectation yields $Q_1=\mathbf 1 \mathbf 1^\top$. Thus the matrix $Q_1$ has rank 1 and is therefore singular.

\section{Proofs of Results on the Hypercube}\label{appendix:hypercube}

\subsection{Proof of Proposition \ref{prop:scperturbation_hypercube_existence}}\label{appendix:scperturbation_hypercube_existence}

\begin{proof}
The function $f$ is non-negative on $\R$ and satisfies $\int_\R f=1$, so it is the density function of some probability distribution over $\R$. Thus, there exists some probability distribution $\cald$ on $\R^d$ such that for $\xi\sim\cald$, the coordinates $\{\xi_i\}_i$ are independent and identically distributed and have probability density function $f$.

We now prove that $\cald$ replicates $\calr$. For all $\theta\in\R^d$ and $i\in[d]$,
\begin{align*}
    \E_{\xi\sim\cald}[\nabla\phi_K(\theta+\xi)]_i
    &= \E[\sgn(\theta_i+\xi_i)]\\
    &=\P(\theta_i+\xi_i\ge 0)- \P(\theta_i+\xi_i\le 0)\\
    &=1-2F_i(-\theta_i)\,,
\end{align*}
where $F_i$ is the cumulative distribution function of $\xi_i$. Notice that $f(t)=\frac{1}{2}L'(t)$. Thus,
$$F_i(t) = \frac{1}{2}(L(t)-L(0))+F_i(0)=\frac{L(t)+1}{2}$$
and
$$\E_{\xi\sim\cald}[\nabla\phi_K(\theta+\xi)]_i=-L(-\theta_i)=L(\theta_i)=\frac{\partial \calr^*(\theta)}{\partial\theta_i}\,,$$
which concludes the proof.
\end{proof}

\subsection{Proof of Proposition \ref{prop:scperturbation_hypercube_bound}}\label{appendix:scperturbation_hypercube_bound}

\begin{proof}
Let $t\in[n]$. For brevity, denote by $a$ the action chosen by \SCFTPL\space at time $t$, $x = \E_{t-1}[a]$ and $Q = \E_{t-1}[aa^\top]$. We also denote $H=\nabla^2\calr^*(-\eta\hat Y_{t-1})$ the matrix such that $\|\cdot\|_t=\|\cdot\|_H$.
\paragraph{Step 1: Computing the expression of $Q^{-1}$.} We first show that $Q$ is non-singular and compute the expression of $Q$. Let $i, j \in[d]$. If $i\ne j$, then $\xi_i$ is independent from $\xi_j$, thus so are $a_i$ and $a_j$ and we have
\[Q_{i,j}=\E_{t-1}[a_ia_j]=\E_{t-1}[a_i]\E_{t-1}[a_j]=x_ix_j\,.\]
If $i=j$, then $Q_{i,i}=\E[a_i^2]$. Because $a\in\operatorname{extr}([-1,1]^d)$, we have that $a_i^2=1$ almost surely, and then $Q_{i,i}=1$. Thus,
\[
    Q = x x^\top  + \Cov(a)\,, \quad \text{where } \Cov(a) = \diag([1 - x_i^2]_{i=1}^d).
\]
A direct computation shows that $Q$ is non-singular, with
\[Q^{-1}=\Big(\frac{-1}{1+\alpha}(\Cov a)^{-1} xx^\top +I_d\Big)(\Cov a)^{-1}\,,\]
with $$\alpha\coloneq x^T(\Cov a)^{-1} x= \sum_{i=1}^d \frac{x_i^2}{(1-x_i^2)}\,.$$

\paragraph{Step 2: Bounding the local norm variance.} Let $y$ denote the loss vector and $\hat y = Q^{-1} a \sca{y}{a}$ the loss vector estimator. It holds that
\begin{align*}
    \E_{t-1}[\|\hat y\|_t^2]
    &= \E_{t-1}[\sca{y}{a}a^\top Q^{-1} HQ^{-1} a \sca{a}{y}]\\
    &\le \E_{t-1}[a^\top Q^{-1} HQ^{-1} a]\\
    &=\tr\left(HQ^{-1}\right)\,,
\end{align*}
where the inequality follows from the boundedness of the losses and the last equality uses the cyclic invariance of the trace. We now develop the expression of $H$, which is
\[H = \nabla^2\calr^*(\nabla\calr(x))=\diag\left[L'(L^{-1}(x_i))\right]_{1\le i\le d}\,,\]
where
$$L(t)=\coth(t)-\frac{1}{t}$$ is the Langevin function, which is a diffeormorphism between $\R$ and $(-1,1)$.
Thus, combining the expression of $H$ and $Q^{-1}$, we get
\begin{align}
    \tr(HQ^{-1})&=\sum_{i=1}^d \left(\frac{-1}{1+\alpha}\frac{x_i^2}{1-x_i^2}+1\right)\frac{1}{1-x_i^2} L'(L^{-1}(x_i))\notag\\
    &\le \sum_{i=1}^d \frac{L'(L^{-1}(x_i))}{1-x_i^2}\label{equation:trHQ-1}
\end{align}
Then, we want to show that
$$g(t)\coloneq \frac{L'(t)}{1-L(t)^2}\le \frac{1}{3}$$
for all $t\in\R$. Because $g$ is even, it suffices to prove the inequality for $t\ge 0$. Developing the expression of $L$, we have that
$$g(t)=\frac{t^2-\sinh^2t}{t^2+\sinh^2t-t\sinh(2t)}\,.$$
For all $t\ge0$, $t\le \sinh t$, so the numerator is nonpositive. Because $g$ is nonnegative, this implies that the denominator is nonpositive, and thus
$$g(t)\le \frac{1}{3} \iff 3(t^2-\sinh^2t)\ge t^2+\sinh^2t-t\sinh(2t)\,.$$
Rearranging terms, we define $h(t)\coloneq 2t^2-4\sinh^2t+t\sinh(2t)$ and aim to show that $h(t)\ge 0$. The successive derivatives of $h$ are:
\begin{align*}
    h'(t)&=4t-3\sinh(2t)+2t\cosh(2t)\\
    h''(t)&=4-4\cosh(2t)+4t\sinh(2t)\\
    h^{(3)}(t)&=-4\sinh(2t)+8t\cosh(2t)\\
    h^{(4)}(t)&=16t\sinh(2t)
\end{align*}
For $t\ge 0$, $h^{(4)}(t)\ge 0$. Since $h^{(3)}(0)=0$, $h^{(3)}(t)\ge 0$ for $t\ge 0$. Similarly, since $h''(0)=0$, $h''(t)\ge 0$ for all $t\ge 0$, and continuing this logic $h'(t)\ge 0$ and finally $h(t)\ge 0$ for all $t\ge 0$. Hence, the inequality $g(t)\le 1/3$ holds for all $t\in\R$. Plunging this inequality into (\ref{equation:trHQ-1}) gives $\tr(HQ^{-1})\le \sum_{i=1}^d 1/3$, and finally
$$\E_{t-1}[\|\hat y\|_t^2]\le \frac{d}{3}$$

\paragraph{Step 3: Bounding the local norm almost surely.} Finally, we derive the almost-sure bound on $\|\hat y\|_t^2$. First, observe that
\begin{align*}
    \|\hat y\|_t^2 &\le a^\top Q^{-1}HQ^{-1}a\\
    &= \sum_{i=1}^d \sum_{j=1}^d a_i [Q^{-1}HQ^{-1}]_{i,j}a_j\\
    &\le \sum_{i,j} |[Q^{-1}HQ^{-1}]_{i,j}|
\end{align*}
Plugging in the closed-form expressions of $Q^{-1}$ and $H$ yields
\begin{align*}
\|\hat y\|_t^2
&\le \frac{1}{(1+\alpha)^2}\left(\sum_{i=1}^d \frac{|x_i|}{1-x_i^2}\right)^2 \sum_{k=1}^d \frac{x_k^2 \,L'(L^{-1}(x_k))}{(1-x_k^2)^2}
\notag\\
&\qquad+\frac{2}{1+\alpha} \left(\sum_{i=1}^d \frac{|x_i|}{1-x_i^2}\right)\sum_{k=1}^d \frac{|x_k| \,L'(L^{-1}(x_k))}{(1-x_k^2)^2}+ \sum_{k=1}^d \frac{L'(L^{-1}(x_k))}{(1-x_k^2)^2}
\end{align*}
Let $C=\sup_{x\in(-1,1)} \frac{L'(L^{-1}(x))}{(1-x^2)^2}\in[0,+\infty]$. We can bound the previous expression by
\begin{equation}\label{eq:proof_prop10_2}
\|\hat y\|_t^2
\le \frac{C}{(1+\alpha)^2}\left(\sum_{i=1}^d \frac{|x_i|}{1-x_i^2}\right)^2 \sum_{k=1}^d x_k^2
+\frac{2C}{1+\alpha} \left(\sum_{i=1}^d \frac{|x_i|}{1-x_i^2}\right)\sum_{k=1}^d |x_k|+ dC
\end{equation}
By the Cauchy–Schwarz inequality,
\[
    \sum_{k=1}^d |x_k|
    \le \sqrt{d \sum_{k=1}^d x_k^2},
    \quad \text{and} \quad
    \sum_{i=1}^d \frac{|x_i|}{1 - x_i^2}
    \le \sqrt{\sum_{i=1}^d \frac{x_i^2}{1 - x_i^2}}
        \sqrt{\sum_{i=1}^d \frac{1}{1 - x_i^2}}
    = \sqrt{\alpha (d + \alpha)}\,.
\]
Plugging theses inequalities into (\ref{eq:proof_prop10_2}) gives
\begin{align*}
\|\hat y\|_t^2&\le \frac{C \alpha(d+\alpha)}{(1+\alpha)^2} \sum_{k=1}^dx_k^2 + \frac{2C}{1+\alpha} \sqrt{d\alpha(d+\alpha)\sum_{k=1}^dx_k^2}+dC
\end{align*}
Moreover, it holds that
\begin{equation*}
(d+\alpha)\sum_{k=1}^dx_k^2 \le d\sum_{k=1}^d \frac{x_k^2}{1-x_k^2}+\alpha\sum_{k=1}^d 1=2\,\alpha d\,,
\end{equation*}
and hence
\begin{equation}
    \|\hat y\|_t^2 \le \frac{2C\alpha^2 d}{(1+\alpha)^2}+\frac{2\sqrt{2}C\alpha d}{1+\alpha}+dC \le (3+2\sqrt{2})dC
\end{equation}
Finally, we bound the value of $C$.
Let
$$g(t)=\frac{L'(t)}{(1-L^2(t))^2}\,$$
for all $t\in\R$. This function is even so it suffices to analyze its behavior on $\R^+$. 
We want to show that for all $t>0$, $g(t)\le 1/2$. Direct computation gives
\[
L'(t) = \frac{1}{t^2} - \frac{1}{\sinh^2 t}, 
\qquad
1 - L(t)^2 = \frac{2 t \cosh t \, \sinh t - t^2 - \sinh^2 t}{t^2 \sinh^2 t}.
\]
Hence, the desired inequality is equivalent to
\[
(1 - L^2)^2 - 2 L' \ge 0.
\]
Clearing the denominator $t^4 \sinh^4 t$ and setting
\[
N(t) := \big(2 t \cosh t \, \sinh t - t^2 - \sinh^2 t\big)^2 
- 2 t^2 \sinh^2 t \, (\sinh^2 t - t^2),
\]
one obtains after a short algebraic rearrangement the identity
\[
N(t) = \Big( t^2 - t \sinh(2t) + \tfrac{1}{2}(\cosh(2t)-1) \Big)^2 
+ 2 t^2 \sinh^2 t \, (\sinh t - t)^2.
\]
Both terms on the right-hand side are nonnegative for $t>0$, so $N(t) \ge 0$ and therefore $(1-L^2)^2 - 2 L' \ge 0$. This yields $g(t) \le 1/2$, as claimed, and thus $C\le 1/2$. Hence we have,
$$\|\hat y\|_t^2  \le (3+2\sqrt{2})\cdot \frac{1}{2}\cdot d \le 3 \, d\,,$$
which concludes the proof.
\end{proof}

\subsection{Proof of Theorem \ref{thm:regret_scftpl_hypercube}} \label{appendix:regret_scftpl_hypercube}

\begin{proof}
Almost surely, it holds that for all $t\in[n]$,
\begin{equation*}
    4\,\eta^2\,\|\hat y_t\|_t^2
    \le 12 \,d\,\eta^2 
= 36 \,d \,\frac{\ln n}{n} \le 1\,,
\end{equation*}
where the first inequality follows from the almost sure bound in Proposition \ref{prop:scperturbation_hypercube_bound}, the equality from the definition of $\eta$ and the last inequality from the theorem's assumption. Thus, the requirements of Theorem \ref{thm:self-conc-linear-bandits} are satisfied, yielding the following regret bound:
\[R_n\le \frac{\vartheta \ln n}{\eta}+\eta \sum_{t=1}^n \E[\|\hat y_t\|_t^2]+2\,,\]
By Proposition \ref{prop:scperturbation_hypercube_bound}, we have $\E[\|\hat y_t\|_t^2]\le d/3$ for all $t$. Furthermore, since the self-concordance parameter $\vartheta$ of $\cald$ equals $d$, it follows that
\[R_n \le \frac{d \ln n}{\eta}+\frac{1}{3}\eta n d+2\,.\]
This bound is optimal for $\eta=\sqrt{\frac{3\ln n}{n}}$, which yields the claimed result.
\end{proof}

\subsection{Complexity Analysis of \SCFTPL\space on the Hypercube}\label{appendix:complexity_analysis_hypercube}

At each round $t\in[n]$, \SCFTPL\space performs three main steps: sampling $\xi_t$ from $\cald$, solving the linear program $a_t=\argmin_{a\in K}\sca{a}{\eta \hat Y_{t-1}-\xi_t}$, and computing the estimator $\hat y_t=Q_t^{-1}a_t\sca{y_t}{a_t}$.

Sampling from $\cald$ given in Proposition \ref{prop:scperturbation_hypercube_existence} reduces to drawing $d$ i.i.d. random variables with probability density function $f$. Using the fact that $f(t)\le 1/(1+t^2)$ for all $t\in\R$, we can use rejection sampling with a Cauchy proposal to sample each coordinate in expected constant time. So the overall sampling step requires $\O(d)$ operations.

The linear program over the hypercube admits the closed-form solution 
$$a_{t,i}=\sgn(-\eta \hat Y_{t-1,i}+\xi_{t,i})$$
for each coordinate $i\in[d]$, which can be computed in constant time per coordinate, yielding $\O(d)$ complexity.

To compute the estimator, we first compute $x_t=\nabla\calr^*(-\eta\hat Y_{t-1})$, which is separable across coordinates and therefore requires $\O(d)$ time. From the proof of Proposition \ref{prop:scperturbation_hypercube_bound}, $Q_t^{-1}a_t$ admits the closed form
\begin{equation*}
    Q_t^{-1}a_t=\Big(\frac{-1}{1+\alpha}\frac{x_{t,i}}{{1-x_{t,i}^2}}\sum_{j=1}^d\frac{x_{t,j}a_{t,j}}{1-x_{t,j}^2}+\frac{a_{t,i}}{{1-x_{t,i}^2}}\Big)_{1\le i\le d}
\end{equation*}
where $\alpha=\sum_{i=1}^d x_{t,i}^2/({1-x_{t,i}^2})$.
This allows computing $Q_t^{-1}a_t$ efficiently: we first evaluate $\alpha$ and 
$$\beta=\sum_{j=1}^d \frac{x_{t,j}a_{t,j}}{1-x_{t,j}^2}$$
in $\O(d)$, and then each coordinate of $Q_t^{-1}a_t$ can be computed in constant time as
$$[Q_t^{-1}a_t]_i = \frac{-\beta}{1+\alpha}\frac{x_{t,i}}{{1-x_{t,i}^2}} +  \frac{a_{t,i}}{{1-x_{t,i}^2}}.$$

Combining all steps, \SCFTPL\space on the hypercube has a total per-round complexity of $\O(d)$.
For comparison, both \scrible\space and \textsc{Osmd} enjoy a per-round complexity in $\O(d)$ when the action set of the hypercube. Indeed, for \scrible, when the action set is the hypercube, we have access to closed forms for the eigenvalues and eigenvectors of $\nabla^2\calr(x_t)$, which allows to attain a per-round complexity in $\O(d)$.

\section{Proofs of Results on the \texorpdfstring{$\ell_2$}{l2} Ball}\label{appendix:ball}

\subsection{Proof of Proposition \ref{prop:scperturbation_ball_existence}}

The proof of Proposition \ref{prop:scperturbation_ball_existence} follows from the following lemma, which gives the expression of the expected value of a multivariate $t$-distribution projected over the unit sphere.
\begin{lemma}\label{lemma:projected_student}
Let $\mathbf T$ be a random vector in $\R^d$ following a multivariate $t$-distribution with location $\mu\in\R^d$, scale matrix $I_d$ and $d+1$ degrees of freedom. Then,
\[\E\left[\frac{\mathbf T}{\|\mathbf T\|}\right] = \frac{\mu}{\sqrt{d+1+\|\mu\|^2}}\,.\]
\end{lemma}

\begin{proof}
\textbf{of Lemma \ref{lemma:projected_student}. }
Let $\mathbf T$ be a random vector following a multivariate $t$-distribution with location $\mu\in\R^d$, scale matrix $I_d$, and $d+1$ degrees of freedom. We use the stochastic representation of the multivariate $t$-distribution \citep{kotz2004multivariate}: for independent random variables $\mathbf N \sim \mathcal{N}(0, I_d)$ and $Z \sim \chi_{d+1}$, the random vector
\[\mu + \frac{\sqrt{d+1}}{Z}\mathbf N\]
has the same distribution as $\mathbf T$. Consequently, the expected projection can be expressed as
\[\E\left[\frac{\mathbf T}{\|\mathbf T\|}\right] = \E\left[\frac{Z\mu+\sqrt{d+1}\mathbf N}{\|Z\mu+\sqrt{d+1}\mathbf N\|}\right]\,.\]
Conditioning on $Z$, the vector $Z\mu + \sqrt{d+1}\mathbf N$ follows a multivariate normal distribution $\mathcal{N}(Z\mu, (d+1)I_d)$. Applying the known characterization for the mean of a projected isotropic normal distribution \citep{herreraesposito2025projectednormaldistributionmoment}, we have:
\[\E\left[\frac{Z\mu+\sqrt{d+1}\mathbf N}{\|Z\mu+\sqrt{d+1}\mathbf N\|}\,\Bigg|\,Z\right]=\frac{\Gamma(\frac{d+1}{2})}{\sqrt{2(d+1)}\,\Gamma(\frac{d}{2}+1)} \;{}_1F_1\left(\frac{1}{2};\frac{d}{2}+1;\frac{-Z^2\|\mu\|^2}{2(d+1)}\right) Z\mu\,,\]
where ${}_1F_1(a;b;z)$ denotes the confluent hypergeometric function. By the tower rule, it follows that
\[\E\left[\frac{\mathbf T}{\|\mathbf T\|}\right] = \mu\,\frac{\Gamma(\frac{d+1}{2})}{\sqrt{2(d+1)}\,\Gamma(\frac{d}{2}+1)}\,\E\left[{}_1F_1\left(\frac{1}{2};\frac{d}{2}+1;\frac{-Z^2\|\mu\|^2}{2(d+1)}\right)Z\right].\]
Given that $Z^2 \sim \chi_{d+1}^2$, we evaluate the remaining expectation in integral form
\begin{align*}
&\E\left[{}_1F_1\left(\frac{1}{2};\frac{d}{2}+1;\frac{-Z^2\|\mu\|^2}{2(d+1)}\right)Z\right]\\&\qquad= \frac{1}{2^{\frac{d+1}{2}}\,\Gamma(\frac{d+1}{2})}\int_0^{+\infty}{}_1F_1\left(\frac{1}{2};\frac{d}{2}+1;-t\frac{\|\mu\|^2} {2(d+1)}\right)\,t^{d/2}\,e^{-t/2}\,dz.
\end{align*}
Using the standard identity for integrals of confluent hypergeometric functions \citep{gradshteyn2007table}, the integral evaluates to
\[\int_0^{+\infty}{}_1F_1\left(\frac{1}{2};\frac{d}{2}+1;-t\frac{\|\mu\|^2} {2(d+1)}\right)\,t^{d/2}\,e^{-t/2}\,dz= \Gamma\left(\frac{d}{2}+1\right)\,2^{\frac{d}{2}+1} \left(1+\frac{\|\mu\|^2}{d+1}\right)^{-1/2}\,.\]
Substituting this result back into the main expression and simplifying the constants, we arrive at the final form
\[\E\left[\frac{\mathbf T}{\|\mathbf T\|}\right] =\frac{\mu}{\sqrt{n+1+ \|\mu\|^2}}\,.\]
\end{proof}

\begin{proof}\textbf{of Proposition \ref{prop:scperturbation_ball_existence}. }
Let $\cald$ be the distribution of $\xi$. By construction, $\cald$ is an absolutely continuous distribution on $\R^d$ with full support. We now show that $\cald$ replicates the 1-self-concordant barrier $x\in\inter\B^d\mapsto -\ln(1-\|x\|^2)$. It suffices to show that for all $\theta\in\R^d$,
\begin{equation*}
    \E_{\xi}[\nabla\phi_{\B^d}(\theta+\xi)]=\nabla\calr^*(\theta)=\theta\cdot \frac{\sqrt{1+\|\theta\|^2}-1}{\|\theta\|^2}\,.
\end{equation*}
For all $x\in\R^d\,\backslash\,\{0\}$, the differential of the support function of $\B^d$ in $x$ is
\[\nabla\phi_{\B^d}(x)=\argmax_{\|a\|\le 1} \sca{a}{x}= \frac{x}{\|x\|}\,.\]
Thus,
\[\E[\nabla\phi_{\B^d}(\theta+\xi)]= \E\left[\frac{\theta+\xi}{\|\theta+\xi\|}\right]=\E\left[\frac{\sqrt{d+1}\,U\theta+\mathbf T}{\|\sqrt{d+1}\,U\theta+\mathbf T\|}\right]\,.\]
Conditionally to $U$, $\sqrt{d+1}\,U\theta+\mathbf T$ follows a multivariate $t$-distribution with location $\sqrt{d+1}\,U\theta$, scale matrix $I_d$ and $d+1$ degrees of freedom. Hence, by Lemma \ref{lemma:projected_student},
\[\E\left[\frac{\sqrt{d+1}\,U\theta+\mathbf T}{\|\sqrt{d+1}\,U\theta+\mathbf T\|}\,\Bigg|\,U\right]=\frac{\sqrt{d+1}\,U\,\theta}{\sqrt{d+1+(d+1)U^2\|\theta\|^2}}= \frac{U\,\theta}{\sqrt{1+U^2 \|\theta\|^2}}\]
Finally, by tower rule,
\begin{align*}
    \E[\nabla\phi_{\B^d}(\theta+\xi)]&= \E\left[\frac{U\,\theta}{\sqrt{1+U^2 \|\theta\|^2}}\right]\\
    &=\theta\,\int_0^1 \frac{u}{\sqrt{1+\|\theta\|^2u^2}}du\\
    &=\theta\, \frac{\sqrt{1+\|\theta\|^2}-1}{\|\theta\|^2}\,,
\end{align*}
which concludes the proof.
\end{proof}

\subsection{Proof of Proposition \ref{prop:scperturbation_ball_bound}}\label{appendix:scperturbation_ball_bound}

\begin{proof}
Let $t\in[n]$. For the sake of brevity, we denote $\theta=-\eta\,\hat Y_{t-1}$ and $\xi=\xi_t$. The action chosen at time $t$ by \SCFTPL\space is then $a=\nabla\phi_{\B^d}(\theta+\xi)$, and its covariance matrix is
\[Q=\E_{t-1}[aa^\top] = \E\left[ \frac{(\theta+\xi)(\theta+\xi)^\top}{\|\theta+\xi\|^2}\right]\,.\]
The expectation is with respect to $\xi$, conditionally on the $t-1$ first rounds. We also denote $y=y_t$ the loss vector at time $t$ and $\hat y=Q^{-1}aa^\top y$ the estimator. Finally, we recall that the norm $\|\cdot\|_t$ is the energy norm relatively to the matrix $\nabla^2\calr^*(\theta)$.
\paragraph{Proof Outline.}
In Step 1, we use the rotational symmetry of the distribution of $\xi$ to show that $Q$ belongs to the linear span of $(P_\theta, \Pperp)$, where $P_\theta$ and $\Pperp$ are the orthogonal projection matrices onto $\theta$ and $\theta^\perp$, respectively. 
In Step 2, we develop the expression of the eigenvalue $q_\perp$ corresponding to $\Pperp$ and show that it can be reduced to a one-dimensional integral expression. Subsequently, in Step 3 we derive lower bounds on both $q$ and $q_\perp$, the eigenvalues corresponding to $P_\theta$ and $\Pperp$, respectively, using the integral expression of $q_\perp$ and an affine relation between $q$ and $q\perp$. Step 4 provides the explicit expression for $Q^{-1}$ and for the Hessian $\nabla^2\calr^*(\theta)$. In Steps 5, 6, and 7, we combine the explicit expressions for $Q^{-1}$ and $\nabla^2\calr^*(\theta)$ with the bounds on $q$ and $q_\perp$ to derive the claimed bounds on the estimator $\hat y$. Finally, in Step 8, we analyze the degenerate case $\theta=0$, where $P_\theta$ is not well-defined, and show that the inequalities still hold.
\paragraph{Step 1: Proving that $Q\in\operatorname{span}(P_\theta, P_{\theta^\perp})$.}
By the stochastic representation of the multivariate $t$-distribution, for independent random variables $\mathbf N\sim\mathcal N(0,I_d)$, $Z\sim\chi_{d+1}$, and $U\sim\mathcal{U}([0,1])$, the random vector
${\mathbf N}/{(Z\,U)}$ follows the same distribution as $\xi$. So
\begin{equation}\label{eq:proof_G2_eq1}
Q = \E\left[ \frac{(ZU\theta+\mathbf N)(ZU\theta+\mathbf N)^\top}{\|ZU\theta+\mathbf N\|^2}\right].
\end{equation}
Conditioning on $Z$ and $U$, $ZU\theta+\mathbf N$ follows a multivariate normal distribution with mean $ZU\theta$ and isotropic covariance. 
Due to the rotational symmetry of this normal distribution around the axis defined by $ZU\theta$, it is a known result \citep{herreraesposito2025projectednormaldistributionmoment} that
\[\E\left[ \frac{(ZU\theta+\mathbf N)(ZU\theta+\mathbf N)^\top}{\|ZU\theta+\mathbf N\|^2}\Bigg|\,Z,U\right] \in \operatorname{span}((ZU\theta)(ZU\theta)^\top, I_d) = \operatorname{span}(\theta\theta^\top,I_d)\,.\]
Then, by tower rule and linearity of the expectation,
\[Q = \E\left[\E\left[ \frac{(ZU\theta+\mathbf N)(ZU\theta+\mathbf N)^\top}{\|ZU\theta+\mathbf N\|^2}\Bigg|\,Z,U\right]\right] \in\operatorname{span}(\theta\theta^\top,I_d). \]
Let
\[P_\theta\coloneq \frac{\theta\theta^\top}{\|\theta\|^2}\,,\quad\text{and}\quad P_{\theta^\perp}\coloneq I_d - P_\theta\]
be the orthogonal projection matrices onto $\theta$ and $\theta^\perp$, respectively. Then, $\spa(\theta\theta^\top,I_d)=\spa(P_\theta,\Pperp)$. Hence, there exists $q,q_\perp\in\R$ such that
\begin{equation}\label{eq:proof_G2_eq2}
    Q=q\,P_\theta+q_\perp\,\Pperp\,.
\end{equation}

\paragraph{Step 2: Developing $q_\perp$.}
We will now develop $q_\perp$ in order to derive upper and lower bounds. First notice that we only need to study one among $q$ an $q_\perp$ as it holds that
\begin{align*}
\tr Q = q\tr (P_\theta) + q_\perp \tr(\Pperp) = q+(d-1)q_\perp
\end{align*}
and
\begin{align*}
\tr Q = \tr\E_{t-1}[aa^\top]=\E_{t-1}[\tr(aa^\top)]=\E_{t-1}[\|a\|^2]=1\,.
\end{align*}
Hence,
\begin{equation}\label{eq:proof_prop10_25}
q=1-(d-1)q_\perp\,.
\end{equation}

Now let $y\in \theta^\perp$ be such that $\|y\|=1$. Multiplying (\ref{eq:proof_G2_eq1}) by $y$ on the left and on the right, it holds that
\[q_\perp=y^\top Q y = \E\left[\frac{(y^\top \mathbf N)^2}{\|ZU\theta+\mathbf N\|^2}\right]\,.\]

To evaluate this expectation, we can always rotate our coordinate system without loss of generality so that $\theta=\|\theta\|e_1$ and $y=e_2$. Then, since $\mathbf N\sim\mathcal N(0,I_d)$ is isotropically distributed, the rotation preserves its distribution. We thus have that that
\begin{equation*}
y^\top \mathbf N = \mathbf N_2
\quad\text{and}\quad ZU\theta+\mathbf N=(ZU\|\theta\|+\mathbf N_1)e_1+\sum_{k=2}^d \mathbf N_k e_k
\end{equation*}
The expectation thus becomes
\begin{equation*}
q_\perp=\mathbb E\Big[\frac{\mathbf N_2^2}{(ZU\|\theta\|+\mathbf N_1)^2 + \sum_{k=2}^d \mathbf N_k^2}\Big]\,.
\end{equation*}
Denote $D$ the denominator in the previous expression. To handle it, we use the integral identity
\begin{equation*}
\frac{1}{D} = \int_0^{+\infty} e^{-vD}dv\,.
\end{equation*}
Conditioned on $Z, U$, we have that
\begin{align*}
\E\Big[\frac{\mathbf N_2^2}{D}\,\Big|\,Z, U\Big]
&= \E\Big[\int_0^{+\infty} e^{-v(ZU\|\theta\|+\mathbf N_1)^2}\, \mathbf N_2^2 e^{-v\mathbf N_2^2} \,\prod_{k=3}^d e^{-v\mathbf N_k^2} \,dv\Big]\\
&=\int_0^{+\infty} \E[e^{-v(ZU\|\theta\|+\mathbf N_1)^2}] \,\E[N_2^2 e^{-v\mathbf N_2^2}]\prod_{k=3}^d \E[e^{-v\mathbf N_k^2}] \,dv
\end{align*}
where the second line comes from the independence of the coordinates of $\mathbf N$. Evaluating each of these standard Gaussian expectation yields
\begin{equation*}
\E\Big[\frac{\mathbf N_2^2}{D}\,\Big|\,Z, U\Big]
= \int_0^{+\infty} (1+2v)^{-d/2-1} \exp\Big(-\frac{vZ^2U^2\|\theta\|^2} {1+2v}\Big)dv\,.
\end{equation*}
Making the change of variable $s=\frac{2v}{1+2v}$, we can simplify the integral to
\begin{equation*}
\E\Big[\frac{\mathbf N_2^2}{D}\,\Big|\,Z, U\Big]
= \frac{1}{2}\int_0^1 (1-s)^{d/2-1} e^{-\frac{1}{2}Z^2U^2\|\theta\|^2s } ds
\end{equation*}
Then, by tower rule,
\begin{equation*}
q_\perp=\frac{1}{2}\int_0^1 (1-s)^{d/2-1} \E[e^{-\frac{1}{2}Z^2U^2\|\theta\|^2s}]\,ds
\end{equation*}
Conditioning on $U$, $Z\sim\chi^2_{d+1}$ so by expression of its moment-generating function,
\begin{equation*}
\E[e^{-\frac{1}{2}Z^2U^2\|\theta\|^2s}|U] = (1+U^2\|\theta\|^2s)^{-(d+1)/2}
\end{equation*}
which gives
\begin{equation*}
    q_\perp= \frac{1}{2}\int_0^1 (1-s)^{d/2-1} \int_0^1 (1+u^2\|\theta\|^2 s)^{-(d+1)/2}\,du\,ds
\end{equation*}
These integrals are converging and using properties of generalized hypergeometric functions \citep{gradshteyn2007table}, one can show that
\begin{align*}
q_\perp 
&= \frac{1}{2}\int_0^1 (1-s)^{d/2-1}\,{}_2F_1\Big(\frac{1}{2},\frac{d+1}{2};\frac{3}{2};-\|\theta\|^2s\Big)\,ds\\
&= \frac{1}{d}\,_3F_2 \Big(\frac{1}{2},1,\frac{d+1}{2};\frac{3}{2},\frac{d}{2}+1;-\|\theta\|^2\Big)\notag\\
&=\frac{\Gamma(\frac{d}{2}+1)}{d\sqrt{\pi}\Gamma(\frac{d}{2})}\int_0^1 s^{\frac{d+3}{2}}(1-s)^{-\frac{1}{2}} \,{}_2F_1\Big(\frac{1}{2},1;\frac{3}{2};-\|\theta\|^2s\Big)\,ds\,.
\end{align*}
Finally, using the closed-form expression of ${}_2F_1(\frac{1}{2},1;\frac{3}{2};\cdot)$ and the change of variable $s=\sin^2\phi$, we get
\begin{align}\label{eq:proof_prop10_26}
q_\perp= \frac{\Gamma(\frac{d}{2})}{\|\theta\|\sqrt{\pi}\Gamma(\frac{d+1}{2})} \int_0^{\pi/2}\sin^{d-1}(\phi) \arctan(\|\theta\| \sin \phi) \,d\phi\,.
\end{align}

\paragraph{Step 3: Bounding $q_\perp$.}
We now want to lower-bound both $q$ and $q_\perp$, using the integral expression of $q_\perp$ in \eqref{eq:proof_prop10_26}.
First, notice that for all $x\ge 0$, $\arctan x \le x$. Thus, we have that
\begin{align*}
q_\perp
&\le \frac{\Gamma(\frac{d}{2})}{\sqrt{\pi}\Gamma(\frac{d+1}{2})} \int_0^{\pi/2}\sin^{d}(\phi) \,d\phi\,.
\end{align*}
Evaluating the integral, we get
\begin{equation*}
q_\perp \le \frac{1}{d}\,.
\end{equation*}
Now, using Equation~\eqref{eq:proof_prop10_25} linking $q$ and $q_\perp$,
\begin{equation}
q\ge 1-\frac{d-1}{d}=\frac{1}{d}\,.
\end{equation}
Second, to lower-bound $q_\perp$, note that for all $x\ge0$, $\arctan x \ge \frac{x}{1+x}$. Thus,
\begin{align*}
q_\perp
&\ge \frac{\Gamma(\frac{d}{2})}{\sqrt{\pi}\Gamma(\frac{d+1}{2})} \int_0^{\pi/2}\frac{\sin^{d}(\phi)}{1+\|\theta\|\sin\phi} \,d\phi\,.
\end{align*}
The denominator of the integrand is smaller than $1+\|\theta\|$. Hence,
\begin{equation}
q_\perp \ge \frac{1}{d(\|\theta\|+1)}\,.
\end{equation}

\paragraph{Step 4: Expression of $Q^{-1}$ and $\nabla^2\calr^*(\theta)$.} We recall that we have
$$Q= q\,P_\theta + q_\perp\,\Pperp.$$
In Steps 2 and 3, we have shown that $q$ and $q_\perp$ are positive. Moreover, because the matrices $P_\theta$ and $\Pperp$ satisfy the following identities
\[P_\theta^2=P_\theta,\qquad\Pperp^2=\Pperp,\qquad P_\theta\Pperp=\Pperp P_\theta=0,\qquad P_\theta+\Pperp=I_d,\]
we have that
$$Q^{-1}= q^{-1}\,P_\theta + q_\perp^{-1}\,\Pperp.$$
For the Hessian of $\calr^*$ in $\theta$, we recall that $\nabla\calr^*(\theta)$ is a radial function with magnitude $\frac{\sqrt{1+\|\theta\|^2}-1}{\|\theta\|}$. Thus,
\begin{align*}
    \nabla^2\calr^*(\theta) = \frac{\sqrt{1+\|\theta\|^2}-1}{\|\theta\|^2 \sqrt{1+\|\theta\|^2}}\,P_\theta + \frac{\sqrt{1+\|\theta\|^2}-1}{\|\theta\|^2}\,\Pperp.
\end{align*}
\paragraph{Step 5: Bounding $\E[\|\hat y\|_t^2]$.}
By definition of $\hat y$,
$$\hat y=Q^{-1}a\sca{y}{a}.$$
Because the loss $\sca{y}{a}$ is bounded by 1 almost surely, we have that
\begin{align*}
    \E[\|\hat y\|_t^2] &= \E[\sca{y}{a}a^\top Q^{-1}\nabla^2\calr^*(\theta)Q^{-1}a\sca{y}{a}]\\
    &\le \E[a^\top Q^{-1}\nabla^2\calr^*(\theta)Q^{-1}a]\\
    &= \tr(Q^{-1}\nabla^2\calr^*(\theta)).
\end{align*}
Plugging the expressions of $Q^{-1}$ and $\nabla^2\calr^*(\theta)$ into this last inequality, we get
\begin{align*}
\E[\|\hat y\|_t^2] &\le q^{-1} \frac{\sqrt{1+\|\theta\|^2}-1}{\|\theta\|^2 \sqrt{1+\|\theta\|^2}} \,\tr(P_\theta) + q_\perp^{-1} \frac{\sqrt{1+\|\theta\|^2}-1}{\|\theta\|^2}\,\tr(\Pperp)\\
&\le d \frac{\sqrt{1+\|\theta\|^2}-1}{\|\theta\|^2 \sqrt{1+\|\theta\|^2}} + d(d-1)(\|\theta\|+1) \frac{\sqrt{1+\|\theta\|^2}-1}{\|\theta\|^2}
\end{align*}
We now need to bound the functions
\[f(t)= \frac{\sqrt{1+t^2}-1}{t^2 \sqrt{1+t^2}}\,\quad\text{and}\quad g(t) = \frac{(t+1)(\sqrt{1+t^2}-1)}{t^2}\,,\]
for all $t\ge 0$.
Define
\[h(t)= t^2\sqrt{1+t^2}-2(\sqrt{1+t^2}-1)\,,\quad t\ge 0.\]
The derivative of $h$ is
\[h'(t) = \frac{3t^3}{\sqrt{1+t^2}} \ge 0.\]
Thus, for all $t\ge 0$, $h(t)\ge h(0)=0$. Thus, by rearranging the terms, for all $t\ge 0$,
$$f(t)\le \frac{1}{2}.$$
To bound $g(t)$, a routine calculation shows
$$g'(t) = \frac{t^3(t+1-\sqrt{1+t^2})}{t^4\sqrt{1+t^2}(\sqrt{1+t^2}+1)}\,.$$
The denominator is always non-negative. For the numerator, we split into two cases based on the sign of $t$. If $t\ge 0$, $t^3\ge 0$ and $\sqrt{1+t^2} \le 1+t$, so the numerator is non-negative. Else, if $t\le 0$, $\sqrt{1+t^2}\ge 1$ so $t+1-\sqrt{1+t^2}\le t$, and $t^3\le 0$, so the numerator is greater than $t^4$. Hence it is non-negative. So $g'(t)\ge 0$ for all $t\ge 0$. So $g$ is non-decreasing on $\R$. Moreover, $\lim_{t\to+\infty} g(t)=1$, so for all $t\in\R$, $g(t)\le 1$.
Thus, we conclude
\begin{equation*}
\E[\|\hat y\|^2_t] \le \frac{1}{2}d + d(d-1) \le d^2\,.
\end{equation*}
This inequality implies the stated inequality in $\frac{5}{4}d^2$. 

\paragraph{Step 6: Bounding $\|\hat y\|_t^2$ almost surely.}
Because the loss $\sca{y}{a}$ is bounded by 1 almost surely, we have that
\begin{align*}
    \|\hat y\|_t^2 &\le a^\top Q^{-1}\nabla^2\calr^*(\theta)Q^{-1}a\\
    &= a^\top\left(q^{-2}\frac{\sqrt{1+\|\theta\|^2}-1}{\|\theta\|^2 \sqrt{1+\|\theta\|^2}} \,P_\theta + q_\perp^{-2} \frac{\sqrt{1+\|\theta\|^2}-1}{\|\theta\|^2}\,\Pperp\right)a.
\end{align*}
For any projection $P$, we can bound $a^\top Pa$ with $a^\top Pa=\|Pa\|^2 \le \|a\|^2=1$. Thus, applying the bounds on $q$ and $q_\perp$,
\begin{align*}
\|\hat y\|_t^2 
&\le d^2 \frac{\sqrt{1+\|\theta\|^2}-1}{\|\theta\|^2 \sqrt{1+\|\theta\|^2}} + d^2 (\|\theta\|+1)^2\frac{\sqrt{1+\|\theta\|^2}-1}{\|\theta\|^2}\\
&\le d^2 f(\|\theta\|) + d^2(\|\theta\|+1) g(\|\theta\|)\,
\end{align*}
where $f$ and $g$ have been defined above. Using the bounds on $f$ and $g$ obtained previously, we get
$$\|\hat y\|_t^2 \le \frac{1}{2}d^2+ d^2\left(\|\theta\|+1\right) \le d^2\|\theta\| + \frac{3}{2}d^2\,,$$
which implies the stated bound.
\paragraph{Step 7: Bounding $\E[\|\hat y\|_2^2]$.}
We bound the variance of the $\ell_2$ norm of $\hat y$ by
\begin{align*}
\E[\|\hat y\|_2^2] &= \E[\sca{y}{a}a^\top Q^{-1} Q^{-1} a \sca{y}{a}]\\
&\le\E[a^\top Q^{-1} Q^{-1} a]\\
&=\tr(Q^{-1}).
\end{align*}
Using the expression of $Q^{-1}$ and the bounds on $q$ and $q_\perp$, we bound the trace of $Q^{-1}$ by
\begin{align*}
    \tr(Q^{-1}) &= q^{-1} \,\tr(P_\theta)+q_\perp^{-1}\,\tr(\Pperp)\\
    &\le d+d(\|\theta\|+1)(d-1).
\end{align*}
Hence,
$$\E[\|\hat y\|_2^2]\le d^2\|\theta\|^2+2d^2.$$
\paragraph{Step 8: $\theta=0$ Case.}
Until now we have assumed that $\theta\ne 0$. Indeed, we cannot define $P_\theta=\theta\theta^\top/\|\theta\|^2$ for $\theta=0$. We now show that the required bounds continue to hold in the degenerate case $\theta=0$. Since $\xi$ is spherically symmetric, $a=\xi/\|\xi\|$ follows a uniform distribution on $\S^{d-1}$. Thus, its covariance matrix is
$$Q=\frac{1}{d}I_d.$$
A direct computation shows that the Hessian of $\calr^*$ in 0 is
$$\nabla^2\calr^*(0)=\frac{1}{2}I_d.$$
Then, $\E[\|\hat y_t\|_t^2]$ simplifies to
$$\E[\|\hat y\|_t^2] = \frac{d^2}{2} \E[y^\top aa^\top aa^\top y].$$
Since $a^\top a=1$, $\E[aa^\top aa^\top]=Q$ and we have
$$\E[\|\hat y\|_t^2] = \frac{d}{2}\|y\|^2\le \frac{d}{2}.$$
For the almost-sure bound, we have similarly
$$\|\hat y\|_t^2 = \frac{d^2}{2} (y^\top a)^2\le \frac{d^2}{2}.$$
And for the $\ell_2$ norm bound,
$$\E[\|\hat y\|_2^2] = d^2\,\E[y^\top aa^\top aa^\top y]\le d.$$
Since $\theta=0$, the above bounds are consistent with the general inequalities.
\end{proof}

\subsection{Proof of Theorem \ref{thm:regret_scftpl_ball}}\label{appendix:regret_scftpl_ball}

\begin{proof}
\paragraph{Step 1: Establishing that $\eta\| \hat y_t\|_t \le 1/2$ holds with high-probability.}
Define the event
\[E\coloneq \left\{\forall t\in[n], \eta \,\|\hat y_t\|_t\le1/2\right\}.\]
From Proposition \ref{prop:scperturbation_ball_bound}, we can bound the local norm of the estimator almost-surely by $\|\hat y_t\|_t^2 \le d^2\,\eta\,\|\hat Y_{t-1}\|+4d^2$ almost surely. It follows that
\[E \supseteq \Big\{\forall t\in[n],\, \|\hat Y_{t-1}\| \le \frac{1}{4\,d^2\,\eta^3}-\frac{4}{\eta}\Big\}.\]
Using the expression of $\eta$ and the assumption that $\frac{n}{\ln n}\ge 64$,
\begin{align}\label{eq:th5__1}
    E&\supseteq\Big\{\sup_{0\le t\le n-1} \|{\hat Y_{t}\|}^2 \le \frac{5}{4}d^2\frac{n}{\ln n}\Big(\frac{5}{16}\frac{n}{\ln n}-4\Big)^2\Big\}
    \notag\\
    &\supseteq \Big\{\sup_{0\le t\le n-1} \|{\hat Y_{t}\|}^2 \le \frac{5}{64} d^2 \Big(\frac{n}{\ln n}\Big)^3\Big\}\,.
\end{align}
For all $t\in[n]$, we can bound $\|\hat Y_t\|^2$ in expectation using the third inequality of Proposition \ref{prop:scperturbation_ball_bound}, which yields
\begin{align*}
\E_{t-1}[\|\hat Y_t\|_2] &=\|{\hat Y_{t-1}\|}^2+\E_{t-1}[\|{\hat y_t\|}^2]+2\sca{y_t}{\hat Y_{t-1}} \\
    &\le \|\hat Y_{t-1}\|^2+d^2\eta\,\|\hat Y_{t-1}\|+2d^2+2\|\hat Y_{t-1}\|.
\end{align*}
Using the assumption that $n/\ln n \ge 2d^2$, we have $d^2\eta\le \sqrt{2/5}$ and $2d^2 \le n$, so the bound becomes
\[\E_{t-1}[\|\hat Y_t\|_2]\le\|\hat Y_{t-1}\|^2+\left(2+\sqrt{\frac{2}{5}}\right)^2\,\|\hat Y_{t-1}\|+n.\]
By the AM-GM inequality, this trinomial can be bounded by
\begin{align}\label{eq:th5__2}
\E_{t-1}[\|\hat Y_t\|_2]&\le \left(1+\frac{1}{n}\right)\|\hat Y_{t-1}\|^2 + \frac{1}{4}\left(2+\sqrt{\frac{2}{5}}\right)^2 n+n \notag\\
&\le \left(1+\frac{1}{n}\right)\|\hat Y_{t-1}\|^2+3n.
\end{align}
If we now define the positive, adapted, stochastic process
\[M_t\coloneq  \left(1+\frac{1}{n}\right)^{-t} \left(\|\hat Y_t\|^2+3n^2\right),\]
then (\ref{eq:th5__2}) bounds the conditional expectancy of $M_t$ by
\begin{align*}
    \E_{t-1}[M_t] \le \Big(1+\frac{1}{n}\Big)^{-t}\Big(\Big(1+\frac{1}{n}\Big)\|\hat Y_{t-1}\|^2+3n+3n^2\Big)=M_{t-1}\,.
\end{align*}
Hence, $(M_t)_t$ is a supermartingale, adapted to the filtration $\{\mathcal{F}_t\}_t$. Therefore, by Ville's inequality, for all $\varepsilon\ge 0$,
\begin{equation*}
    \P\Big[\sup_{0\le t\le n-1} M_t\ge \varepsilon\Big]\le \frac{\E[M_0]}{\varepsilon} = \frac{3n^2}{\varepsilon}\,.
\end{equation*}
Moreover, for all $t\in[n]$, it holds that
\begin{equation*}
    \|\hat Y_t\|^2=\left(1+\frac{1}{n}\right)^t M_t-3n^2 \le e\,M_t\,.
\end{equation*}
Combining this last inequality with (\ref{eq:th5__1}), we get
\begin{align*}
    E\supseteq \Big\{\sup_{0\le t\le n-1} M_t \le \frac{5}{64\,e}d^2\Big(\frac{n}{\ln n}\Big)^3\Big\}\,.
\end{align*}
Hence,
\begin{align}\label{eq:th5__4}
    \P(E^c)\le\P\Big[\sup_{0\le t <n}\|\hat Y_t\|^2 > \frac{5}{64}d^2\Big(\frac{n}{\ln n}\Big)^3\Big]\le \frac{192\,e}{5}\frac{(\ln n)^3}{d^2\,n}\,.
\end{align}

\paragraph{Step 2: Bounding the regret $R_n$.}
In order to bound the regret using similar arguments as in Theorem \ref{thm:self-conc-linear-bandits}, we introduce an appropriate stopping time and invoke Doob’s optional stopping theorem to bound the regret conditionally on $E$. Define the random variable
\begin{equation*}
    \tau\coloneq\inf\Big\{t<n:\|\hat Y_t\|^2 > \frac{5}{64}d^2\Big(\frac{n}{\ln n}\Big)^3\Big\}\wedge n\,.
\end{equation*}
Because $(\hat Y_t)_t$ is a stochastic process adapted to the filtration $(\mathcal F_t)_t$, $\tau$ is a stopping time with respect to the same filtration.
For all $u\in K$, we can decompose the regret of \SCFTPL\space w.r.t $u$ as
\begin{equation}\label{eq:th5__5}
    R_n(u) = \E\Big[\sum_{t=1}^\tau \sca{y_t}{a_t-u}\Big]+\E\Big[\sum_{t=\tau+1}^n\sca{y_t}{a_t-u}\Big]\,.
\end{equation}
We now proceed to bound each term of the right-hand side of (\ref{eq:th5__5}) separately. For the second term, we have
\begin{align}\label{eq:th5__6}
    \E\Big[\sum_{t=\tau+1}^n\sca{y_t}{a_t-u}\Big] &\le \P(\tau <n)\,\E\Big[\sum_{t=1}^n|\sca{y_t}{a_t-u}|\Big|\tau<n\Big]\notag\\
    &\le \P\Big[\sup_{0\le t <n}\|\hat Y_t\|^2 > \frac{5}{64}d^2\Big(\frac{n}{\ln n}\Big)^3\Big]\;2n\notag\\
    &\le \frac{384\,e}{5\,d^2}\,\ln ^3 n\,,
\end{align}
where the second inequality follows from the boundedness of the loss vectors $y_t$ and the last inequality follows from (\ref{eq:th5__4}). For the first term, let
\begin{equation*}
    B_t\coloneq \sum_{s=1}^t \sca{y_s}{a_s-u}-\sca{\hat y_s}{x_s-u}\,.
\end{equation*}
The process $(B_t)_t$ is adapted to the filtration $(\mathcal F_t)_t$ and for all $t\in[n]$,
\begin{align*}
    \E_{t-1}[B_t]=B_{t-1}+\sca{y_t}{\E_{t-1}[a_t-u]} - \sca{\E_{t-1}[\hat y_t]}{x_t-u}=B_{t-1}\,.
\end{align*}
Therefore, $(B_t)_t$ is a martingale with $B_0=0$. Thus, by Doob's optional stopping theorem, $\E[B_\tau]=0$, i.e.,
\begin{equation*}
    \E\Big[\sum_{t=1}^\tau \sca{y_t}{a_t-u}\Big]=\E\Big[\sum_{t=1}^\tau\sca{\hat y_t}{x_t-u}\Big]\,.
\end{equation*}
Moreover, for all $t\le \tau$, we have that $\|\hat Y_{t-1}\|^2\le \frac{d^2}{8}(\frac{n}{\ln n})^3$, and thus $\eta\|\hat y_t\|_t\le \frac{1}{2}$. Thus, we can bound $\sum_{t=1}^\tau\E[\sca{\hat y_t}{x_t-u}]$ following the arguments from the proofs of Lemma \ref{lemma:bandits_gbpa} and of Theorem \ref{thm:self-conc-linear-bandits}, which yield
\begin{equation}\label{eq:th5__7}
    \E\Big[\sum_{t=1}^\tau\sca{\hat y_t}{x_t-u}\Big]\le \frac{\vartheta \ln n}{\eta} + \eta \sum_{t=1}^n \E[\|\hat y_t\|_t^2]+2\,.
\end{equation}
Note that Theorem \ref{thm:self-conc-linear-bandits} could not have been applied directly, as we needed the bound $\eta\|\hat y_t\|_t\le \frac{1}{2}$ to hold almost-surely. Also, conditioning the expectation in the regret definition by the event $E$ would not have enabled to use the unbiasedness of $a_t$ and $\hat y_{t}$, which justifies the need for the stopping time argument we have developed. Combining (\ref{eq:th5__6}) and (\ref{eq:th5__7}), we have that
\begin{align*}
    R_n(u) &\le \frac{\vartheta \ln n}{\eta} + \eta \sum_{t=1}^n \E[\|\hat y_t\|_t^2]+2 +\frac{384\,e}{5}\,\frac{\ln^3n}{d^2}\\
    &\le \frac{\ln n}{\eta}+\frac{5}{4}\,\eta\,n\,d^2+2+\frac{384\,e}{5}\,\frac{\ln^3n}{d^2}\,,
\end{align*}
where the last inequality follows from the value of the self-concordance parameter $\vartheta=1$ and from the bound on $\E[\|\hat y_t\|_t^2]$ of Proposition \ref{prop:scperturbation_ball_bound}. Finally, plugging the value of $\eta$ and taking the supremum over all $u\in K$, we have
\begin{equation*}
    R_n\le d\sqrt{5 \,n\ln n}+2+\frac{384\,e}{5}\frac{\ln^3n}{d^2}\,.
\end{equation*}
\end{proof}

\subsection{Complexity Analysis of \SCFTPL\space on the $\ell_2$ Ball}\label{appendix:complexity_analysis_ball}

At each round $t\in[n]$, \SCFTPL\space performs theses three main steps: 
\begin{enumerate}
    \item Sampling $\xi_t\sim\cald$,
    \item Solving the linear program $a_t=\argmin_{a\in\B^d}\sca{a}{\eta \hat Y_{t-1}-\xi_t}$, and
    \item Computing the estimator $\hat y_t=Q_t^{-1}a_t\sca{y_t}{a_t}$.
\end{enumerate}
Sampling from $\cald$ given in Proposition \ref{prop:scperturbation_ball_existence} reduces to drawing $\mathbf T$ from a multivariate $t$-distribution with location 0, scale matrix $I_d$ and $d+1$ degrees of freedom, and $U$ from a uniform distribution on $[0,1]$ independently of $\mathbf T$, and to compute
$$\xi_t=\frac{1}{\sqrt{d+1}\,U}\mathbf T.$$
By the stochastic representation of the multivariate $t$-distribution \citep{kotz2004multivariate}, sampling $\mathbf T$ amounts to sampling $\mathbf N\sim \mathcal N(0,I_d)$ and $Z\sim \chi_{d+1}$ independently and computing $$\mathbf T= \frac{\sqrt{d+1}}{Z}\mathbf N.$$
Sampling from a standard multivariate normal distribution can be done in $\O(d)$ operations, while sampling from a $\chi_{d+1}$ distribution or from an uniform distribution on $[0,1]$ can be done in constant times. Hence, the overall sampling step requires $\O(d)$ time.

The linear program over the Euclidean ball admits the closed-form solution
$$a_t = \left[\frac{-\eta \hat Y_{t-1,i}+ \xi_{t,i}}{\|-\eta \hat Y_{t-1}+ \xi_{t}\|}\right]_{i=1}^d.$$
The norm of $-\eta \hat Y_{t-1}+ \xi_{t}$ can be computed in $\O(d)$ operations, after which each coordinate of $a_t$ can be computed in constant time, yielding overall $\O(d)$ complexity.

For the estimator, from the proof of Proposition~\ref{prop:scperturbation_ball_bound}, we have that
$$Q_t^{-1} = q^{-1}\,P_{\theta}+q_\perp^{-1}\,\Pperp,$$
where
\begin{equation*}
\theta=-\eta\hat Y_{t-1}
,\qquad
P_\theta=\frac{\theta\theta^\top}{\|\theta\|^2}
,\qquad
\Pperp = I_d-P_\theta
\end{equation*}
and $q$ and $q_\perp$ are given by
\begin{equation*}
q=1-(d-1)q_\perp
,\qquad
q_\perp = \frac{\Gamma(\frac{d}{2})}{\|\theta\|\sqrt{\pi}\Gamma(\frac{d+1}{2})} \int_0^{\pi/2}\sin^{d-1}(\phi) \arctan(\|\theta\| \sin \phi) \,d\phi\,.
\end{equation*}
A closed-form expression for $q_\perp$ would involve the generalized hypergeometric function ${}_3F_2$, for which there exists no general numerical evaluation method. Instead, this integral expression can be computed using Gauss-Legendre quadrature. Because the integrand in analytic, this method yields exponential convergence.

Once $q_\perp$ has been numerically computed, $q$ follows and
$$Q_t^{-1}a_t = q_\perp^{-1} a_t + (q^{-1}-q_\perp^{-1})\,\frac{\sca{a_t}{\theta}}{\|\theta\|^2}\theta.$$
So $Q_t^{-1}a_t$ can be computed in $\O(d)$ time knowing $q$ and $q_\perp$.

Combining all steps, the per-round computational complexity of \SCFTPL\space on the $\ell_2$-ball is $\mathcal O(d)$.

\end{document}